%% file: RCM_new.tex
\documentclass[aoas,preprint]{imsart}\usepackage[]{graphicx}\usepackage[]{color}
\makeatletter
\def\maxwidth{ %
  \ifdim\Gin@nat@width>\linewidth
    \linewidth
  \else
    \Gin@nat@width
  \fi
}
\makeatother

\definecolor{fgcolor}{rgb}{0.345, 0.345, 0.345}

\usepackage{framed}
\makeatletter
 {\par\unskip\endMakeFramed%
 \at@end@of@kframe}
\makeatother

\definecolor{shadecolor}{rgb}{.97, .97, .97}
\definecolor{messagecolor}{rgb}{0, 0, 0}
\definecolor{warningcolor}{rgb}{1, 0, 1}
\definecolor{errorcolor}{rgb}{1, 0, 0}

\usepackage{alltt}

\RequirePackage[OT1]{fontenc}
\RequirePackage{amsthm,amsmath}
\RequirePackage{natbib}
\RequirePackage[colorlinks,citecolor=blue,urlcolor=blue]{hyperref}

\startlocaldefs
\numberwithin{equation}{section}
\theoremstyle{plain}

\endlocaldefs

\usepackage{amssymb}   
\usepackage{dsfont}    
\usepackage{hyperref}  
\usepackage{longtable} 
\usepackage{lscape}    

\usepackage{algorithm}
\usepackage{algpseudocode}

\theoremstyle{plain}

\usepackage{thmtools, thm-restate} 

\renewcommand{\algorithmicrequire}{\textbf{Input:}}
\renewcommand{\algorithmicensure}{\textbf{Output:}}

\usepackage{subcaption}
\usepackage{float}
\usepackage{chngcntr}

\DeclareMathOperator*{\argmax}{arg\,max}
\DeclareMathOperator*{\tr}{tr}

\renewcommand{\vec}{\boldsymbol}

\newcommand{\bbE}{\mathbb{E}}

\newcommand{\calN}{\mathcal{N}}
\newcommand{\calW}{\mathcal{W}}
\newcommand{\calS}{\mathcal{S}}
\newcommand{\bbOne}{\mathds{1}}
\newcommand{\var}{\text{Var}}
\newcommand{\cov}{\text{Cov}}

\newcommand{\vSigma}{{\vec{\Sigma}}}
\newcommand{\hvSigma}{{\hat{\vec{\Sigma}}}}
\newcommand{\vPsi}{{\vec{\Psi}}}

\newcommand{\vDelta}{{\vec{\Delta}}}
\newcommand{\vTheta}{{\vec{\Theta}}}
\newcommand{\hvTheta}{{\hat{\vec{\Theta}}}}
\newcommand{\hvPsi}{{\hat{\vec{\Psi}}}}

\newcommand{\vX}{{\vec{X}}}
\newcommand{\vx}{{\vec{x}}}
\newcommand{\vY}{{\vec{Y}}}
\newcommand{\vS}{{\vec{S}}}

\newcommand{\vI}{{\vec{I}}}
\newcommand{\vE}{{\vec{E}}}

\IfFileExists{upquote.sty}{\usepackage{upquote}}{}
\begin{document}

\begin{frontmatter}
\title{Estimating a common covariance matrix for network meta-analysis of
gene expression datasets in diffuse large B-cell lymphoma\thanksref{T1}}
\runtitle{Estimating a common covariance matrix}
\thankstext{T1}{Supported by MSCNET, EU FP6, CHEPRE, the Danish Agency for Science, Technology, and Innovation as well as Karen Elise Jensen Fonden.}

\begin{aug}
\author{\fnms{Anders Ellern} \snm{Bilgrau}\thanksref{m1,m2,m3},
\ead[label=e1]{anders.ellern.bilgrau@gmail.com}
}
\author{\fnms{Rasmus Froberg \snm{Br\o{}ndum}}\thanksref{m2,m3},
\ead[label=e5]{rfb@rn.dk}
}
\author{\fnms{Poul Svante} \snm{Eriksen}\thanksref{m1},
\ead[label=e2]{svante@math.aau.dk}
}
\author{\fnms{Karen}~\snm{Dybk\ae{}r}\thanksref{m2},
\ead[label=e3]{k.dybkaer@rn.dk}
}
\and
\author{\fnms{Martin} \snm{B\o{}gsted}\thanksref{m1,m2}
\ead[label=e4]{mboegsted@dcm.aau.dk}
}
\runauthor{A.\ Bilgrau et al.}

\affiliation{Aalborg University\thanksmark{m1} and Aalborg University Hospital\thanksmark{m2}\\ Shared first authorship\thanksmark{m3}}

\address{
Department of Haematology\\
Sdr.\ Skovvej 15\\
DK-9000 Aalborg\\
\printead{e1}\\
\phantom{E-mail: }\printead*{e5}\\
\phantom{E-mail: }\printead*{e4}\\
\phantom{E-mail: }\printead*{e3}
}

\address{
Department of Mathematical Sciences\\
Fredrik Bajers Vej 7G\\
DK-9220 Aalborg \O{}\\
\printead{e1}\\
\phantom{E-mail: }\printead*{e2}
}

\address{
Department of Clinical Medicine\\
Sdr.\ Skovvej 15\\
DK-9000 Aalborg \O{}\\
\printead{e4}\\
\phantom{E-mail: }\printead*{e3}
}

\end{aug}

\begin{abstract}
The estimation of covariance matrices of gene expressions has many applications in cancer systems biology. Many gene expression studies, however, are hampered by low sample size and it has therefore become popular to increase sample size by collecting gene expression data across studies.
Motivated by the traditional meta-analysis using random effects models, we present a hierarchical random covariance model and use it for the meta-analysis of gene correlation networks across 11 large-scale gene expression studies of diffuse large B-cell lymphoma (DLBCL).
We suggest to use a maximum likelihood estimator for the underlying common covariance matrix and introduce an EM algorithm for estimation.  By simulation experiments comparing the estimated covariance matrices by cophenetic correlation and Kullback-Leibler divergence the suggested estimator showed to perform better or not worse than a simple pooled estimator.  In a posthoc analysis of the estimated common covariance matrix for the DLBCL data we were able to identify novel biologically meaningful gene correlation networks with eigengenes of prognostic value.
In conclusion, the method seems to provide a generally applicable framework for meta-analysis, when multiple features are measured and believed to share a common covariance matrix obscured by study dependent noise.
\end{abstract}


\end{frontmatter}

\section{Introduction}
Human cells carry out their function in concerted interaction via intricate protein signalling networks. These networks are according to the central dogma of molecular biology controlled by expressed genes.  It has become popular to perform genome wide measurements of expressed genes and proteins and summarizing the information by huge covariance matrices leading to improved understanding of disease pathology and identification of new drug targets \citep{Agnelli2011, Clarke2013}. Many gene expression studies, however, are hampered by low sample size and it has therefore become of interest to increase sample size by collecting gene expression data across studies. These data are potentially hampered by severe batch effects, and robust methods are therefore required to conduct meta-analysis of covariance matrices.

To the best of our knowledge no approaches exist where meta-analysis of covariance matrices have been addressed explicitly. We acknowledge, however, that a number of indirect methods have been constructed. An immediate and tempting approach is to use one of the many study correcting approaches scattered around in the literature \citep{Irizarry2003, Johnson2007, Lee2014} followed by estimating the covariance matrix either based on a pooled data set or by pooling covariance matrices estimated from each individual study as suggested by \cite{Lee2014}. This approach, however, suffers from the same disadvantages as usual meta-analysis based on pooling fixed effects as it puts too much weight on large outliers in the data \citep{Borenstein2010}.

Motivated by the alternative meta-analysis by random effects \citep{DerSimonian1986, Choi2003}, we suggest a hierarchical model where the covariance for each study is assumed to be drawn from an inverse Wishart distribution with a common mean covariance matrix, and data from each study is then subsequently generated from a multivariate Gaussian distribution with this covariance matrix. We suggest to use a maximum likelihood estimator for the underlying common covariance matrix and introduce an EM algorithm for its estimation. We use the method for the meta-analysis of gene correlation networks across 11 large-scale gene expression studies of diffuse large B-cell lymphoma (DLBCL). It is our expectation that a more suitable handling of the covariance matrix will lead to more adequate estimations of covariance matrices and subsequently inferred gene correlation networks.

In Section \ref{sec:RCMmodel}, we propose the model for a common covariance matrix across multiple studies, derive estimators thereof, and propose an inter-study homogeneity measure to aid in assessing the variation between studies. We conduct an extensive simulation study in Section \ref{sec:estimationassessment} comparing the proposed estimator and simple pooling of covariance matrices. We then apply the model in Section \ref{sec:DLBCL} to $2{,}046$ DLBCL samples across 11 datasets before concluding the manuscript in Section \ref{sec:conclusion}.

\section{A hierarchical model for the covariance matrix}
\label{sec:RCMmodel}

Let $p$ be the number of features and $k$ the number of studies.
We model an observation $\vx$ from the $i$'th study as a $p$-dimensional zero-mean multivariate Gaussian vector with covariance matrix realized from an inverse Wishart distribution, i.e.\ $\vx$ follows the hierarchical model
\begin{align}
\begin{split} \label{eq:RCM}
  \vSigma_i  &\sim \calW^{-1}_p\big(\vPsi, \nu\big), \\
  \vx | \vSigma_i &\sim \calN_p(\vec{0}_p, \vSigma_i), \qquad i = 1, ..., k,
\end{split}
\end{align}
where $\calN_p(\vec{\mu},\vSigma_i)$ denotes a $p$-dimensional multivariate Gaussian distribution with mean $\vec{\mu}$ and positive definite (p.d.) covariance matrix $\vSigma_i$, and probability density function (pdf) shown in \eqref{eq:normalpdf},
and $\calW^{-1}_p(\vPsi, \nu)$ denotes a $p$-dimensional inverse Wishart distribution with $\nu > p - 1$ degrees of freedom, a p.d. $p \times p$ scale matrix $\vPsi$, and pdf
shown in \eqref{eq:wishartpdf}.
While the inverse Wishart distribution is defined for all $\nu > p - 1$, the first order moment exists only when $\nu > p + 1$ and is given by
\begin{align}
  \label{eq:expcovar}
  \bbE[\vSigma_i] = \vSigma = \frac{\vPsi}{\nu-p-1} \text{ for } \nu > p + 1.
\end{align}
Hence, in the Random Covariance Model (RCM) of \eqref{eq:RCM}, $\vSigma$ can be interpreted as a location-like parameter as it is the expected covariance matrix in each study.
The parameter $\nu$ inversely controls the inter-study variation and can as such be considered an inter-study homogeneity parameter of the covariance structure.
A large $\nu$ corresponds to high study homogeneity and vice versa for small $\nu$.
This can further be seen as $\vSigma_i$ concentrates around $\vSigma$ for $\nu\to\infty$ which corresponds to a vanishing inter-study variation for increasing $\nu$.
This fact is seen directly from variance and covariance expressions for the inverse Wishart (see \eqref{eq:invwishcovar} and \eqref{eq:invwishvar}) where the 4th order denominator grows much faster than the 1st order nominator as polynomials in $\nu$ and causing the variance to vanish for $\nu\to\infty$.
Thus, the true underlying covariance matrix $\vSigma$ and the homogeneity parameter $\nu$ are the effects of interest to be estimated.

\subsection{The likelihood function}
Suppose $\vx_{i1}, \dots,\vx_{in_i}$ are $n_i$ i.i.d.\ observations from $i = 1,...,k$ independent studies from the model given in \eqref{eq:RCM}.
Let $\vX_i = (\vx_{i1}, \dots,\vx_{in_i})^\top$ be the $n_i \times p$ matrix of observations for the $i$'th study where rows correspond to samples and columns to variables.
By the independence assumptions, the log-likelihood for $\vPsi$ and $\nu$ is given by
\begin{align*}
  &\ell\!\left(\vPsi, \nu \big|\vX_1, ..., \vX_k  \right)
  = \log f\!\left(\vX_1, ..., \vX_k \big| \vPsi, \nu \right) \\
  &= \log\!\int
             f(\vX_1, ...,\vX_k |
               \vSigma_1, ..., \vSigma_k, \vPsi, \nu)
             f(\vSigma_1, ..., \vSigma_k | \vPsi, \nu)
             \mathrm{d}\vSigma_1 \cdots \mathrm{d}\vSigma_k \\
  &= \log \prod_{i=1}^k \!\int
               f(\vX_i | \vSigma_i)
               f(\vSigma_i | \vPsi, \nu)
               \mathrm{d}\vSigma_i.
\end{align*}
Throughout, we use the generic notation $f(\cdot | \cdot)$ and $f(\cdot)$ for the conditional and unconditional pdf of random variables, respectively.
Since the inverse Wishart distribution is conjugate to the multivariate Gaussian distribution, the integral---of which the integrand forms a Gaussian-inverse-Wishart distribution---can be evaluated. Hence $\vSigma_i$ can be marginalized out, cf.\ \eqref{eq:marg1} in Appendix \ref{sec:marginalization}, and we arrive at the following expression for the log-likelihood function,
\begin{align}
  &\ell\!\left(\vPsi, \nu \big| \vX_1, ..., \vX_k \right) 
  = \log\prod_{i=1}^k
    \frac{\big|\vPsi\big|^\frac{\nu}{2} \Gamma_p\!\left(\frac{\nu+n_i}{2}\right)}
         {\pi^\frac{n_i p}{2} \big|\vPsi +\vX_i^\top\vX_i\big|^\frac{\nu+n_i}{2}
          \Gamma_p\!\left(\frac{\nu}{2}\right)}          \notag\\
  &= \sum_{i=1}^k \!\bigg[
       \frac{\nu}{2}  \log\big|\vPsi\big|
       - \frac{\nu + n_i}{2}\log\big| \vPsi +\vX_i^\top\vX_i \big|
       + \log\frac{\Gamma_p\!\left(\frac{\nu + n_i}{2}\right)}
                  {\Gamma_p\!\left(\frac{\nu}{2}\right)}
       \!\bigg]\!,
    \label{eq:loglik}
\end{align}
up to an additive constant where $\Gamma_p$ is the multivariate generalization of the gamma function $\Gamma$, see \eqref{eq:multigamma}.
The scatter matrix $\vS_i =\vX_i^\top\vX_i$ and study sample size $n_i$ are sufficient statistics for each study.
Note that $\vS_i$ is conditionally Wishart distributed, $\vS_i|\vSigma_i \sim \calW(\vSigma_i, n_i)$, by construction.

As stated in the following two propositions, the likelihood is not log-concave in general. However, it is log-concave as a function of $\nu$. All proofs have been deferred to Appendix \ref{sec:proofs}.

\begin{restatable}[Non-concavity in $\vPsi$]{proposition}{propositionNonConcavityInPsi}
  \label{prop:nonconcavityinpsi}
  For a fixed $\nu$, the log-likelihood function \eqref{eq:loglik} is not
  concave in $\vPsi$.
\end{restatable}

\begin{restatable}[Concavity in $\nu$]{proposition}{propositionConcavityInNu}
  \label{prop:concavityinnu}
  For a fixed positive definite $\vPsi$, the log-likelihood function \eqref{eq:loglik}
  is concave in $\nu$.
\end{restatable}

\noindent While the likelihood function is not concave in $\vPsi$ we are able to show the existence and uniqueness of a global maximum in $\vPsi$.

\begin{restatable}[Existence and uniqueness]{proposition}{propositionUniqueMax}
\label{prop:uniquemax}
The log-likelihood \eqref{eq:loglik} has a unique maximum in $\vPsi$ for fixed $\nu$ and $n_\bullet = \sum_{a=1}^k n_a \geq p$.
\end{restatable}

%
%

In the following section estimators of the parameters are derived using moments and the EM algorithm assuming $\nu$ to be fixed.

\subsection{Moment estimator}
The pooled empirical covariance matrix can be viewed as a moment estimator of $\vSigma$.
By the model assumptions, the first and second moment of the $j$'th observation in the $i$'th study, $\vx_{ij}$, is given by $\bbE[\vx_{ij}] = \vec{0}_p$ and
\begin{align*}
  \bbE[\vx_{ij}\vx_{ij}^\top]
    &= \bbE\!\left[ \bbE[ \vx_{ij}\vx_{ij}^\top | \vSigma_i ]\right]
    = \bbE[\vSigma_i]
    = \frac{\vPsi}{\nu - p - 1}
    = \vSigma.
\end{align*}
for all $j = 1, ..., n_i$ and $i = 1, ..., k$. This suggests the estimators
\begin{align}
  \label{eq:pooledest}
  \hvPsi_\text{pool}
  = (\nu - p - 1)\frac{\sum_{i = 1}^k \vS_i}{\sum_{i = 1}^k n_i}
  \text{ and }
  \hvSigma_\text{pool}
  = \frac{\sum_{i = 1}^k \vS_i}{\sum_{i = 1}^k n_i}, \qquad \nu > p + 1
\end{align}
where the latter is obtained by plugging $\hvPsi_\text{pool}$ into \eqref{eq:expcovar}.
This is the well-known pooled empirical covariance matrix.

\subsection{Maximization using the EM algorithm}
Here the updating scheme of the expectation-maximization (EM) algorithm \citep{Dempster1977} for fixed $\nu$ is derived.
We now compute the expectation step of the EM-algorithm.

From \eqref{eq:RCM} we have that,
\begin{align*}
  \vSigma_i          &\sim \calW^{-1}_p\big(\vPsi, \nu\big), \\
  \vS_i | \vSigma_i  &\sim \calW_p(\vSigma_i, n_i) \quad \text{ for } i = 1, ..., k.
\end{align*}
Let $\vDelta_i = \vSigma_i^{-1}$ be the precision matrix and let $\vTheta = \vPsi^{-1}$, then we equivalently have that
\begin{align}
  \vDelta_i
  &\sim \calW_p\big(\vTheta, \nu\big),
  \notag\\
  \vS_i | \vDelta_i
  &\sim \calW_p( \vDelta_i^{-1}, n_i).
  \label{eq:precisiondensity}
\end{align}
From the conjugacy of the inverse Wishart and the Wishart distribution, the posterior distribution of the precision matrix is
\begin{align*}
    \vDelta_i | \vS_i
    &\sim \calW_p\!\Big( \big(\vTheta^{-1} + \vS_i\big)^{-1}, n_i + \nu\Big).
\end{align*}
Hence, by the expectation of the Wishart distribution,
\begin{align*}
  \bbE[\vDelta_i |\vS_i] = (n_i + \nu)\big(\vTheta^{-1} + \vS_i\big)^{-1}.
\end{align*}
The maximization step, in which the log-likelihood $\ell(\vTheta|\vDelta_1, ..., \vDelta_k)$ is maximized, yields the estimate
$
 \hat{\vTheta} = \frac{1}{k\nu}\sum_{i = 1}^k \vDelta_i,
$
which is the mean of the scaled precision matrices $\frac{1}{\nu}\vDelta_i$ (derived in  Appendix \ref{sec:precisionloglik}).
Let $\hvTheta_{(t)}$ be the current estimate of $\vTheta$.
This yields the updating scheme
\begin{align}
  \label{eq:em}
  \hvTheta_{(t+1)}
  = \frac{1}{k\nu}\sum_{i = 1}^k
    (n_i + \nu)\left(\hvTheta_{(t)}^{-1} + \vS_i\right)^{-1}
\end{align}
for $\vTheta_{(t)}$.
We denote the inverse of the estimate obtained by repeated iteration of \eqref{eq:em} by $\hvPsi_\text{EM}$. The EM algorithm can be sensitive to starting values. Hence, starting the algorithm in different starting values can help assesing if a global maximum has been reached.

An approximate maximum likelihood estimator using a first order approximation is also possible (derived in Appendix \ref{sec:amle}).

\subsection{Estimation procedure}
We propose a procedure alternating between estimating $\nu$ and $\vPsi$ while keeping the other fixed.
Given parameters $\hat{\nu}_{(t)}$ and $\hvPsi_{(t)}$ at iteration $t$, we estimate $\hvPsi_{(t+1)}$ using fixed $\hat{\nu}_{(t)}$. Subsequently, we find $\hat{\nu}_{(t+1)}$ by a standard one-dimensional numerical optimization procedure using the fixed $\hvPsi_{(t+1)}$.
This coordinate ascent approach is repeated until convergence as described in Algorithm \ref{alg:RCM}.
\begin{algorithm}[tb]
\caption{RCM coordinate ascent estimation procedure}
\label{alg:RCM}
\begin{algorithmic}[1]
\State \algorithmicrequire{
\State \emph{Sufficient data:} $(\vS_1, n_1), ..., (\vS_k, n_k)$
\State \emph{Initial parameters:} $\hvPsi_{(0)}, \hat{\nu}_{(0)}$
\State \emph{Convergence criterion:} $\varepsilon > 0$
}
\State \algorithmicensure{
\State \emph{Parameter estimates:} $\hvPsi, \hat{\nu}$
}
\Procedure{fitRCM}{$\vS_1, ..., \vS_k, n_1, ..., n_k, \hvPsi_{(0)}, \hat{\nu}_{(0)}, \varepsilon$}
  \State \emph{Initialize}: $l_{(0)} \gets \ell(\hvPsi_{(0)}, \hat{\nu}_{(0)})$
  \For {$t = 1, 2, 3, ...$}
    \State {$\hvPsi_{(t)} \gets U\!\left(\hvPsi_{(t-1)}, \hat{\nu}_{(t-1)}\right)$}
    \State {$\hat{\nu}_{(t)} \gets \argmax_\nu \ell\!\left(\hvPsi_{(t)}, \nu\right)$}
    \State {$l_{(t)} \gets \ell\!\left(\hvPsi_{(t)}, \hat{\nu}_{(t)}\right)$}
    \If {$l_{(t)} - l_{(t-1)} < \varepsilon$}
      \State \Return {$\Big(\hvPsi_{(t)}, \nu_{(t)}\Big)$}
    \EndIf
  \EndFor
\EndProcedure
\end{algorithmic}
\end{algorithm}
The update function $U$ in the algorithm is defined by the derived estimators.
That is, equations \eqref{eq:pooledest}, \eqref{eq:em}, or \eqref{eq:mle} define $U$ as the pooled, EM, or approximate MLE estimates, respectively.

The procedure using the EM step utilizes the results about the RCM log-likelihood and thus provides a guarantee of convergence along with the advantage of a very simple implementation.
Both the EM step and the $\nu$ update will always yield an increase in the likelihood.
The disadvantage is that the identified stationary point might be a local maximum or
saddle-point when considering the log-likelihood function jointly in
$(\vPsi, \nu)$.
Intuitively, the latter possibility happens with zero probability, but it cannot be excluded that the maximum found is not global.

Variations on the convergence criterion can also be considered, such as (a) using the difference in successive parameter estimates, or (b) using relative rather than absolute differences.

\subsection[Interpretation and inference of nu]{Interpretation and inference}
\subsubsection*{Intra-study correlation coefficient}
The heterogeneity parameter $\nu$ has no straightforward interpretation partly because the values of $\nu$ which corresponds to a large study heterogeneity is dependent on the dimension $p$.
We therefore introduce a descriptive statistic analogous to the intra-study correlation coefficient (ICC) \citep{Shrout1979} well known from ordinary meta-analysis.
For the RCM this follows from the definition of the ICC which is defined to be the ratio of the between-study variation $\var(\Sigma_{ij})$ and the total variation $\var(S_{ij})$ of any single pair of variables. In Appendix \ref{app:ICC} it is shown that the ICC is given by:
\begin{align}
  \text{ICC}(\nu)
  = \frac{1}{\nu - p}.
  \label{eq:ICCexprs}
\end{align}
The ICC might in this sense be utilized in better quantifying the reproducibility of the covariance across studies.
A straight-forward plug-in estimator $\widehat{\text{ICC}}(\nu)$ of the ICC of some gene-gene interaction is then $\text{ICC}(\hat{\nu})$.

Though $v > p + 3$ is required for the variances to exist, it is clear that
$\text{ICC}(\nu) \to 1$ for $\nu \to (p+1)^+$ and  $\text{ICC}(\nu) \to 0$ for $\nu \to \infty$
as should be expected.

\subsubsection*{Test for no study heterogeneity}
By the RCM $\nu$ parameterizes an inter-study variance where the size of $\nu$ corresponds to the homogeneity between the studies.
A large $\nu$ yields high study homogeneity while a small $\nu$ yields low homogeneity.
Thus, it might be of interest to test if the estimated homogeneity $\hat{\nu}$ is extreme under the null-hypothesis of no heterogeneity (i.e.\ infinite homogeneity).
I.e.\ a test for the hypothesis $H_0: \nu = \infty$ which is equivalent to
\begin{align*}
  H_0: \vSigma_1 = ... = \vSigma_k = \vSigma.
\end{align*}
The two are equivalent since sampling the covariance matrix from the inverse Wishart distribution becomes deterministic for $\nu = \infty$.
Therefore, testing this hypothesis can also be interpreted as testing whether the data is adequately explained when leaving out the hierarchical structure.

The distribution of $\hat{\nu}$ under the null hypothesis is not tractable.
However, in practice under $H_0$ or when $\nu$ is extremely large the estimated $\hat{\nu}_\text{obs}$ will be finite as the intra-study variance dominates the total variance.
We note that the null distribution of $\hat{\nu}$ does not depend on $\vSigma$.
We propose approximating the distribution of $\hat{\nu}$ under $H_0$ by resampling.
To do this, the model is simply fitted a large number of times $N$ on datasets re-sampled under $H_0$ mimicked by permuted study labels to get $\hat{\nu}_0^{(1)}, ..., \hat{\nu}_0^{(N)}$.
As \textit{small} values of $\hat{\nu}$ are critical for $H_0$ approximate acceptance regions can be constructed from $\hat{\nu}_0^{(j)}, j = 1,...,N$. Likewise, an approximation of the $p$~value testing $H_0$ can be obtained by
\begin{align}
  \label{eq:pvalue}
  P = \frac{1}{N+1}
  \Bigg(
    1 + \sum_{j=1}^N \bbOne\!\Big[\hat{\nu}_0^{(j)} < \hat{\nu}_\text{obs}\Big]
  \Bigg),
\end{align}
where $\bbOne[ \,\cdot\, ]$ is the indicator function.
The addition of one to both nominator and denominator adds a positive bias to the approximate p-value and is considered minimally needed according to \citet{Phipson2010}.
This is approximately the fraction of $\hat{\nu}^{(j)}_0$'s smaller than $\hat{\nu}_\text{obs}$.

\subsection{Implementation and availability}
Algorithm \ref{alg:RCM} and the different estimators are implemented in the statistical programming language R \citep{R} with core functions in \texttt{C++} using packages Rcpp and RcppArmadillo \citep{Eddelbuettel2011, RcppArmadillo}.
They are incorporated in the open-source R-package \texttt{correlateR} freely available for forking and editing \citep{correlateR}.
We refer to the information here for further details and installation instructions.
This document was prepared with \texttt{knitr} \citep{Xie2013} and LaTeX.
To reproduce this document see \url{http://github.com/AEBilgrau/RCM}.

\section{Simulation experiments}
\label{sec:estimationassessment}
\subsection{Evalutation of Network Estimation}

To assess the estimation procedures ability to estimate $\Sigma$ we generated data from the hierarchical model \eqref{eq:RCM} in two different scenarios. In the first scenario we define a simple block matrix of dimension $p=40$ with four blocks of size $10$. Each block has an internal pairwise correlation of $0.5$, blocks 1\and 2 and 3\and 4 have a correlation of $0.3$ between all pairs, and the remaing correlations are set at $0.1$. In the second scenario we select the top 100 genes, ranked by variance, from the IDRC dataset (see Table \ref{tab:studies}) and used the scatter matrix of these genes, scaled as a correlation matrix, as the $\Sigma$ matrix for simulation. For both scenarios we performed agglomerative hierchacical clustering using Ward-linkage and 1 minus the absolute correlation as a distance measure. Heatmaps with associated hierarchical clustering of both $\Sigma$ matrices are shown in Supplementary Figure \ref{fig:S1}.

For both scenarios we simulate data with $k=3$ and a range of values for $n_i$ and $\nu$. Each simulation was repeated 100 times, and the correlation matrix was estimated using the EM, MLE, and Pool approaches as outlined in Section \ref{sec:RCMmodel}. The similarity of the estimated and true $\Sigma$ matrices and associated networks were evaluated using respectively the Kullback-Leibler divergence \citep{Mattiussi2011} and the cophenetic correlation \citep{Sokal1962}. The cophenetic correlation is defined as the correlation of cophenetic distances of all pairwise distances in a tree, where the cophenetic distance is the height of the lowest point on the tree where two points merge. Results from the first scenario (EM and Pool method in Table \ref{tab:results.clustering}, full results in Supplementary Table \ref{tab:results.clustering.full}) show that for heterogenous data ($\nu = 50, 100$) and $n_i \geq p$ the EM estimator outperforms the Pool and MLE estimators using both measures. Examples of tanglegrams comparing networks estimated with the EM and Pool method and the true $\Sigma$ matrix are shown in Supplementary Figure \ref{fig:S2}. Tanglegrams were constructed using the R-package \texttt{dendextend} \citep{Galili2015}. Increasing the $\nu$ parameter, thereby making the data more homogeneous across groups diminishes the advantage of the EM estimator. Similar results were found in the second scenario using a $\Sigma$ matrix based on the IDRC dataset (Table \ref{tab:results.clustering.idrc}). Results furthermore showed that the estimates in terms of cophenetic correlation for the MLE and Pool approaches are nearly identical. We expect this to be caused by the fact that the MLE method is initilized with the Pool estimates and stops after few iterations; presumably a better estimate cannot be found in these simple scenarios.

\input{table1}

\subsection{Computation time for the RCM model}

Next we tested the performance of the different methods in terms of computation time. Figure \ref{fig:1} shows computation times of the methods with varying values of the dimension of the data, and demonstrates that the increased performance of the EM method comes at an extra cost in computation time.
\begin{figure}
  \includegraphics[width=0.8\linewidth]{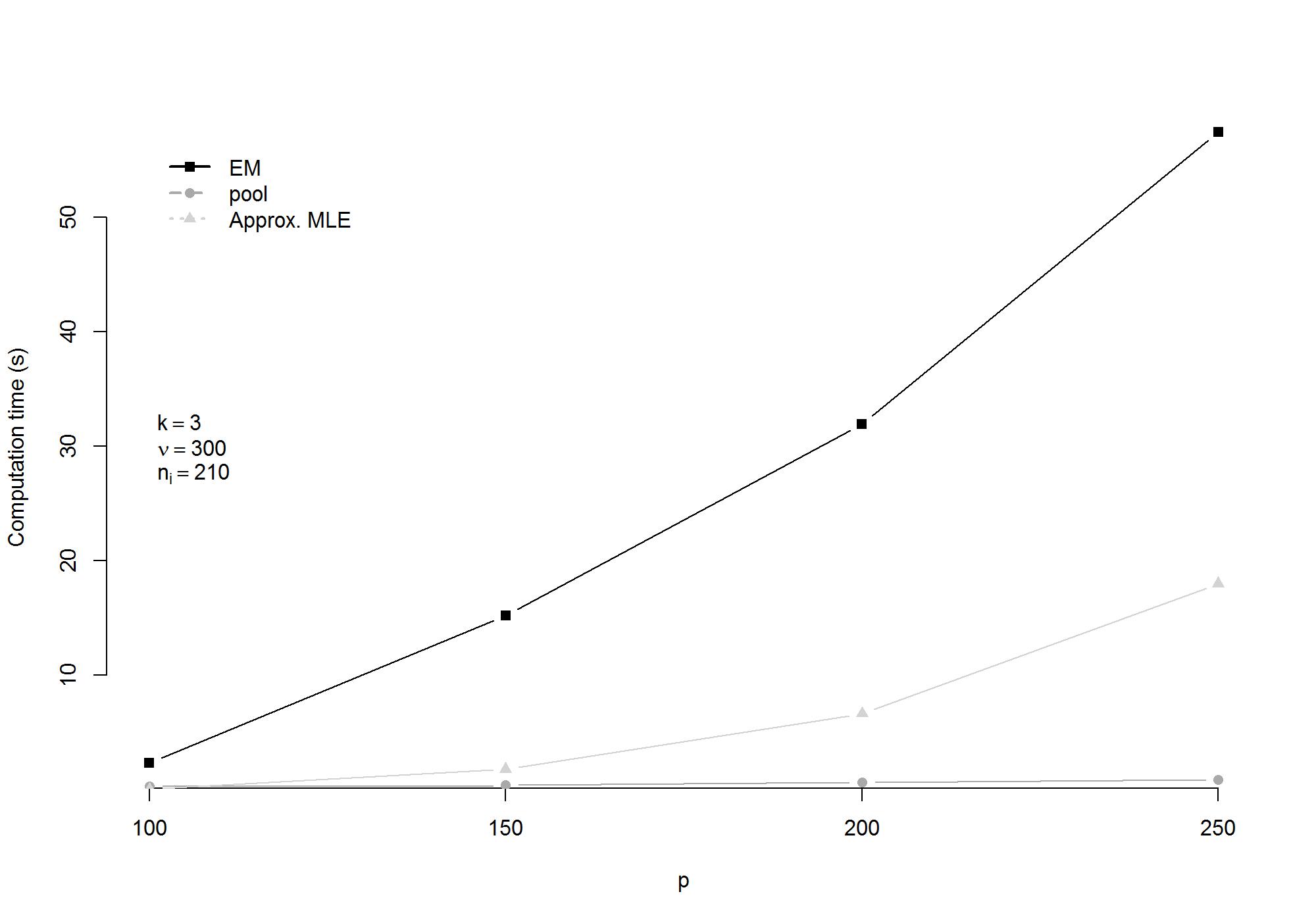}
  \caption{The mean computation time of $10$ fits with varying dimension $p$.}
  \label{fig:1}
\end{figure}

\subsection{Evaluation of the hypothesis testing}

Finally we investigate the performance of the P-value for the hypothesis test suggested in \eqref{eq:pvalue}. To do this, we simulate from the hierarchical model with $k=3$ and a range of different values for $p$, $\nu$, and $n_i$. For these simulations we used a $\Psi$ matrix with a diagonal of ones and $0.5$ for off-diagonal values. Simulations were done 100 times for each scenario, and 500 permutations were done for each simulation. Results summarized as boxplots of the P-values obtained in the 100 simulations for each scenario are shown in Supplementary Figure \ref{fig:S3}. We find that for heterogenous data (e.g. $p=20, \nu=30$) the null-hypothesis is clearly rejected if $n_i>p$. When increasing $\nu$ thus making the groups more similar, more observations are needed to reject the null hypothesis, while for identical groups, i.e. $\nu=\infty$, the null-hypothesis is generally not rejected. The P-values obtained from the permuation test thus performs as intended.

\section{DLBCL meta-analysis}
\label{sec:DLBCL}

Diffuse large B-cell lymphoma (DLBCL) is an aggressive cancer subtype accounting for $30\%-58\%$ of non-Hodgkin's lymphomas (NHL) which constitutes about $90 \%$ of all lymphomas \citep{Project1997}.

\subsection{Data and preprocessing}
A large amount of DLBCL gene expression datasets are now available online at the NCBI (National Center for Biotechnology Information) Gene Expression Omnibus (GEO) website.
Ten large-scale DLBCL gene expression studies were downloaded and preprocessed using custom brainarray chip definition files (CDF) \citep{Dai2005} and RMA-normalized using the R-package \texttt{affy} \citep{affy}.
The corresponding GEO-accession numbers and microarray platforms used are seen in Table \ref{tab:studies}.
The downloaded data yield a total of 2046 samples with study sizes in the range 78-469.
The summarization using brainarray CDFs to Ensembl gene identifiers facilitates cross-platform integration.

After RMA normalization and summarization, the data were brought to a common scale by quantile normalizing all data to the common cumulative distribution function of all arrays. Lastly, the datasets were reduced to 11573 common genes represented in all studies and array platforms. Figure \ref{fig:S4} shows a plot of the first and second principal components of the combined dataset. We see a clear split on the first principal component, indicitating a possible batch effect and heterogeneous data, and thus a situation where the EM estimator might offer an advantage compared to the simpler Pool approach.

\begin{table}[!tbp]
\caption{Overview of studies used with GEO accession number from the NCBI Gene expression omnibus website, the relevant reference, array types used in the study, and number of samples and features on the used array.\label{tab:studies}}
\begin{center}
\begin{tabular}{llllp{2cm}ll}
\hline\hline
\multicolumn{1}{l}{}&
\multicolumn{1}{c}{GEO no.}&
\multicolumn{1}{c}{Name}&
\multicolumn{1}{c}{Reference}&
\multicolumn{1}{c}{Used arrays}&
\multicolumn{1}{c}{$n$}\tabularnewline
\hline
1&GSE56315&CHEPRETRO&\cite{DybkaerBoegsted2015}&hgu133plus2&89\tabularnewline
2&GSE19246&BCCA&\cite{Williams2010}&hgu133plus2&177\tabularnewline
3&GSE12195&CUICG&\cite{Compagno2009}&hgu133plus2&136\tabularnewline
4&GSE22895&HMRC&\cite{Jima2010}&hugene10st&101\tabularnewline
5&GSE31312&IDRC&\cite{Visco2012}&hgu133plus2&469\tabularnewline
6&GSE10846&LLMPP R-CHOP&\cite{Lenz2008}&hgu133plus2&181\tabularnewline
7&GSE10846&LLMPP CHOP&\cite{Lenz2008}&hgu133plus2&233\tabularnewline
8&GSE34171&MDFCI&\cite{Monti2012}&hgu133plus2, snp6&90\tabularnewline
9&GSE34171&MDFCI&\cite{Monti2012}&hgu133a, hgu133b&78\tabularnewline
10&GSE22470&MMML&\cite{Salaverria2011}&hgu133a&271\tabularnewline
11&GSE4475&UBCBF&\cite{Hummel2006}&hgu133a&221\tabularnewline
\hline
\end{tabular}\end{center}
\end{table}

\subsection{Analysis}
For each dataset the scatter matrix $\vS_i$ of the top 300 most variable genes (as measured by the pooled variance across all studies) was computed as the sufficient statistics along with the number of samples.

The parameters of the RCM were estimated using the EM algorithm and yielded the $300 \times 300$ matrix $\hvPsi$, $\hat{\nu} = 773.16$, 
and ICC = $0.0021$. The RCM was fitted using three different initial sets of parameters which all converged to the same parameter estimates.
Log-likelihood traces, iterations used, and computation times are seen in Figure \ref{fig:2}.
From the parameter estimate, the common expected covariance $\hvSigma = (\hat{\nu}-p-1)^{-1}\hvPsi$ was computed and subsequently scaled to the corresponding correlation matrix $\hat{\vec{R}}$.

\begin{figure}
  \includegraphics[width=0.8\linewidth]{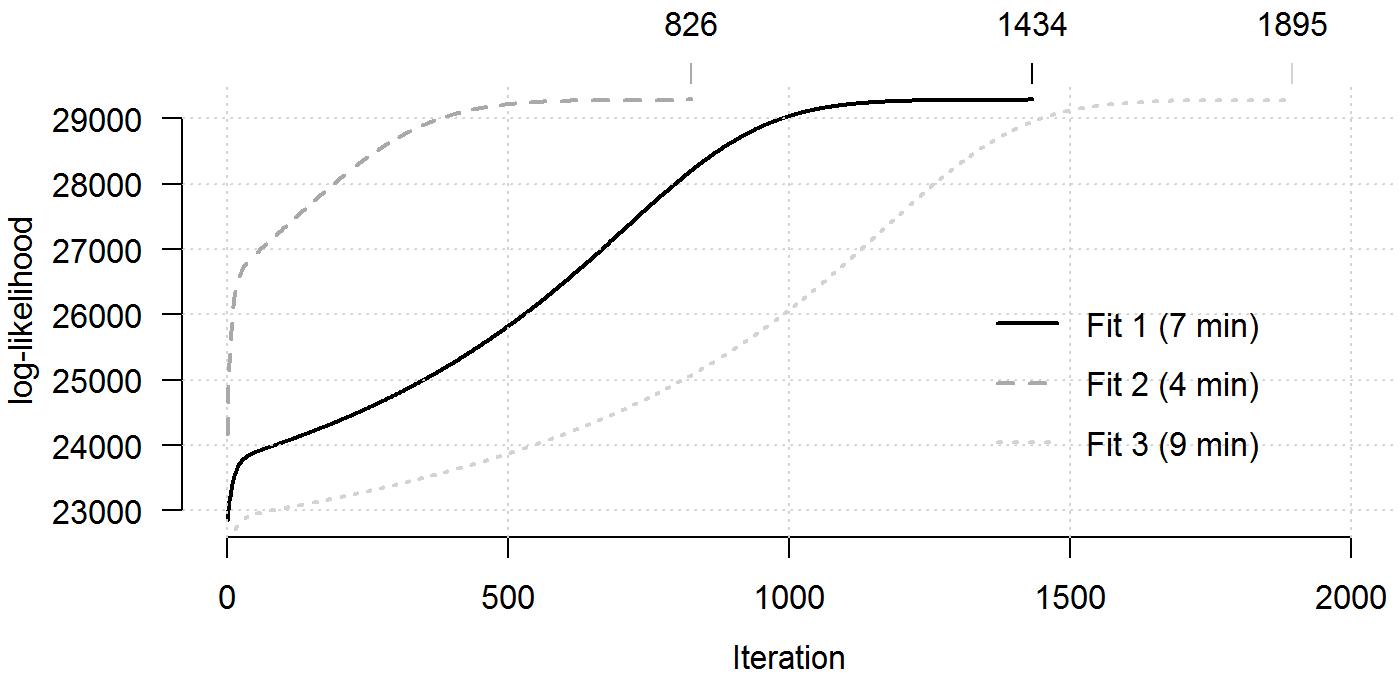}
  \caption{The trace of the log-likelihood for three different starting values of $\vPsi$ and $\nu$ using the EM algorithm and computational times in minutes. The number of iterations used for each fit is shown above.}
  \label{fig:2}
\end{figure}

Despite the low ICC value the permutation test yielded a P-value for the null hypothesis of study homogeneity of $0.002$, clearly rejecting it. This means a significant difference 
has been detected between the estimated covariance structures across studies.
This low ICC might suggest selecting the most variable genes bias the ICC towards inter-study homogeneity of covariances. To further investigate the low ICC value we randomly sampled 300 genes and estimated the $\nu$ parameter 100 times. This gave a value of $\nu$ ranging from 382.69 to 395.18 with a mean of 388.87, corresponding to an ICC ranging from 0.0105 to 0.0121 with a mean of 0.0113; histograms are shown in Supplementary Figure \ref{fig:S7}. This indicates a bias towards more homogeneity for the high variance selected genes.

For simplicity we employed a standard network analysis to the estimated common correlation matrix $\hat{\vec{R}}$ across all studies. To identify clusters with high internal correlation, we used agglomerative hierarchical clustering with Ward-linkage and distance measure defined as 1 minus the absolute value of the correlation. The dendrogram was arbitrarily pruned at a height which produced 5 modules. The Modules are given different colors. Figure \ref{fig:3} shows the heatmap, associated network modules and suggested function.

\begin{figure}
\begin{subfigure}{0.5\textwidth}
  \centering
  \includegraphics[width=\linewidth]{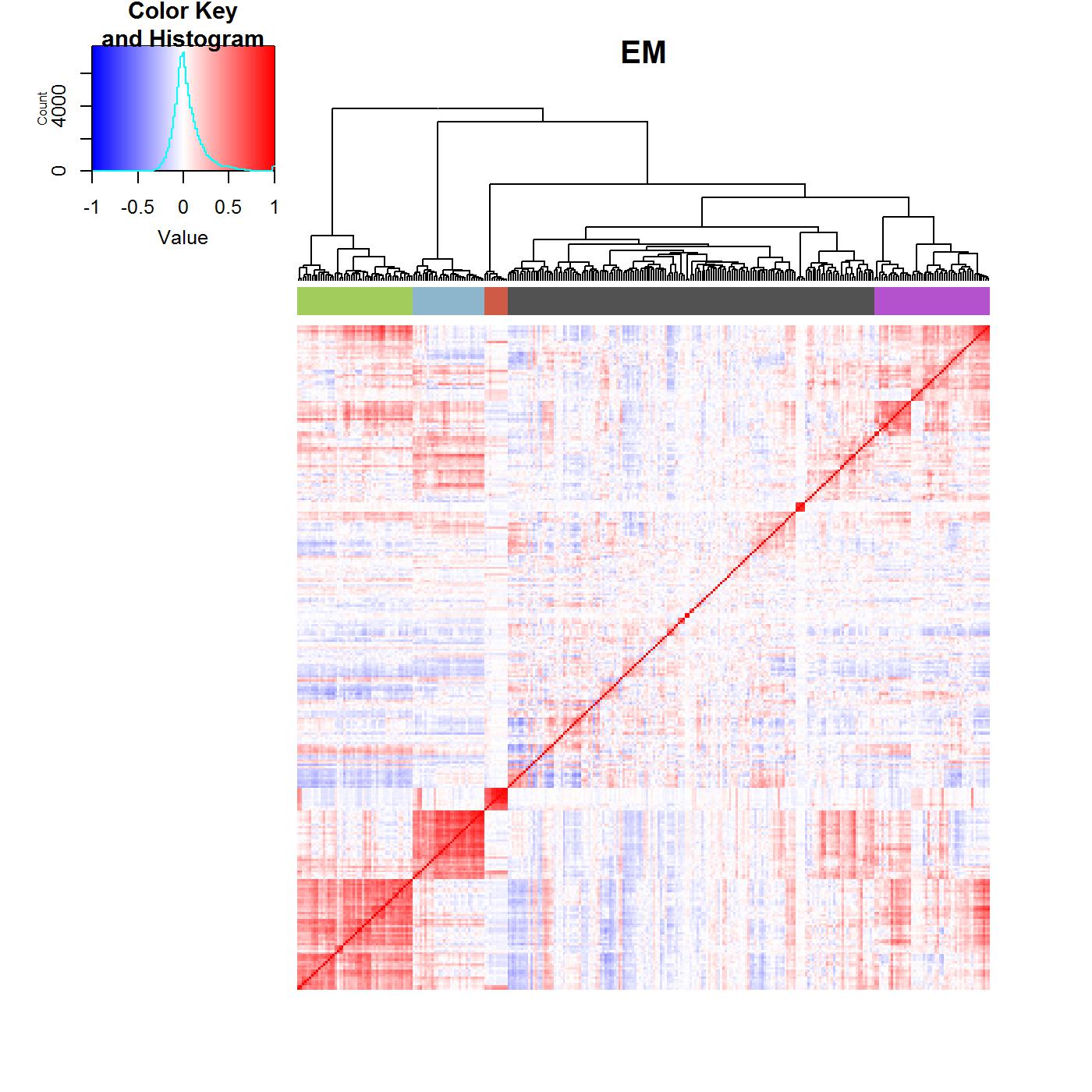}
   \label{fig:3a}
\end{subfigure}%
\begin{subfigure}{0.5\textwidth}
  \centering
  \includegraphics[width=\linewidth]{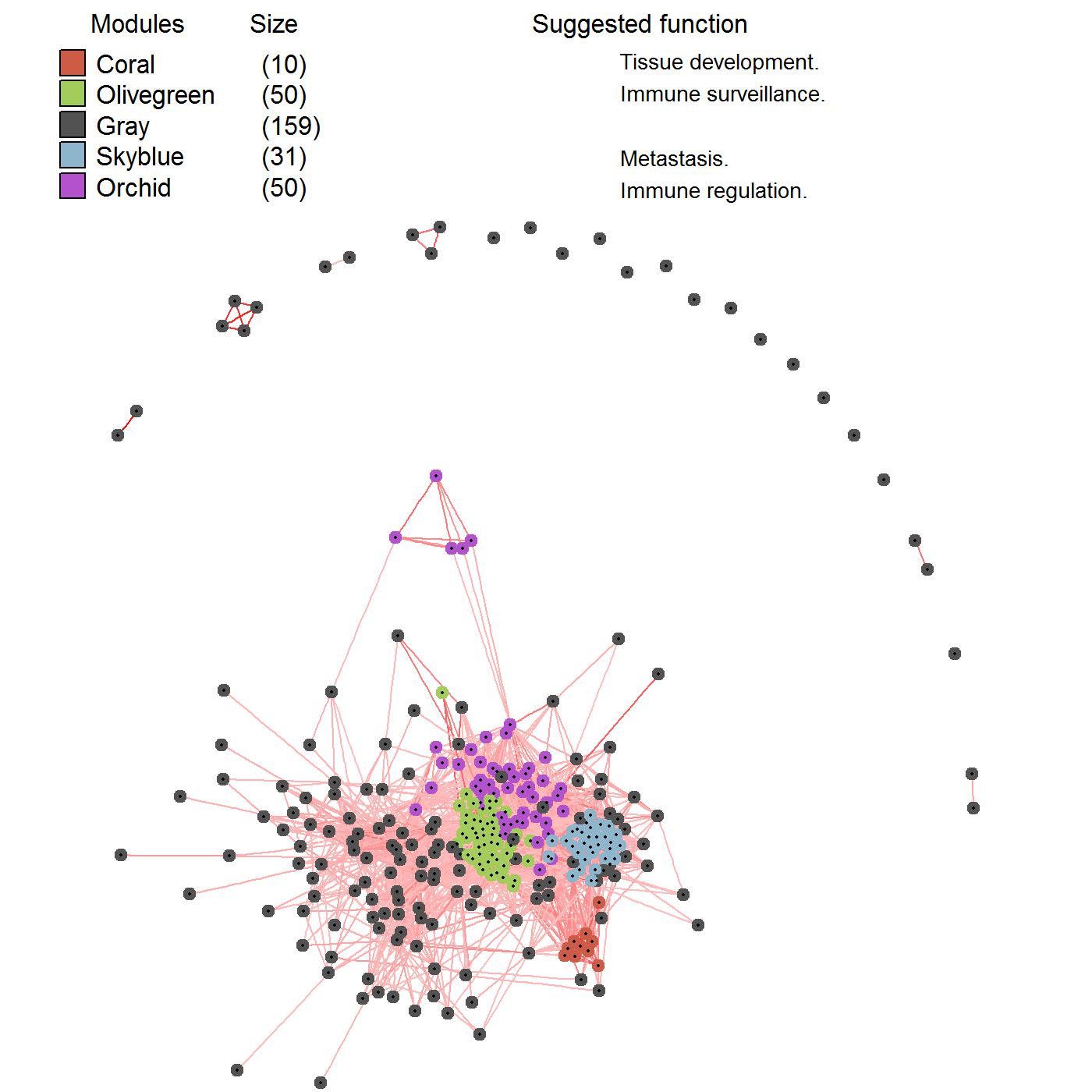}
  \label{fig:3b}
\end{subfigure}
\caption{Heatmap and correlation network for the estimated correlation matrices of the top 300 genes for the DLBCL data using the EM method. The network is cut at a height producing 5 clusters}
\label{fig:3}
\end{figure}

We checked if the identified modules were prognostic for overall survival (OS) in the CHOP and R-CHOP-treated cohort datasets of GSE10846. To do this, the eigengene \citep{Horvath2011} for each module was computed. The module eigengene is the first principal component of the expression matrix of the module which thus can be represented by a linear combination of the module genes. We also report the amount of variation the eigengene represents by calculating the explained variation of the first pricipal component. Multiple Cox proportional hazards model for OS was fitted with the module eigengenes as covariates.  For the prognostically interesting and tightly clustered olivegreen module, the Kaplan-Meier estimates were computed for groups arising when dichotomizing the values of the corresponding eigengene as above or below the median value. These results are shown in Figure \ref{fig:4}. The proportion of variance explained by the eigengene in the CHOP and R-CHOP datasets for respectively the Coral, Olivegreen, Gray, Skyblue and Orchid modules were 0.72, 0.6, 0.11, 0.7, 0.31, and 0.77, 0.55, 0.11, 0.7, 0.31.

\begin{figure}
  \includegraphics[width=0.8\linewidth]{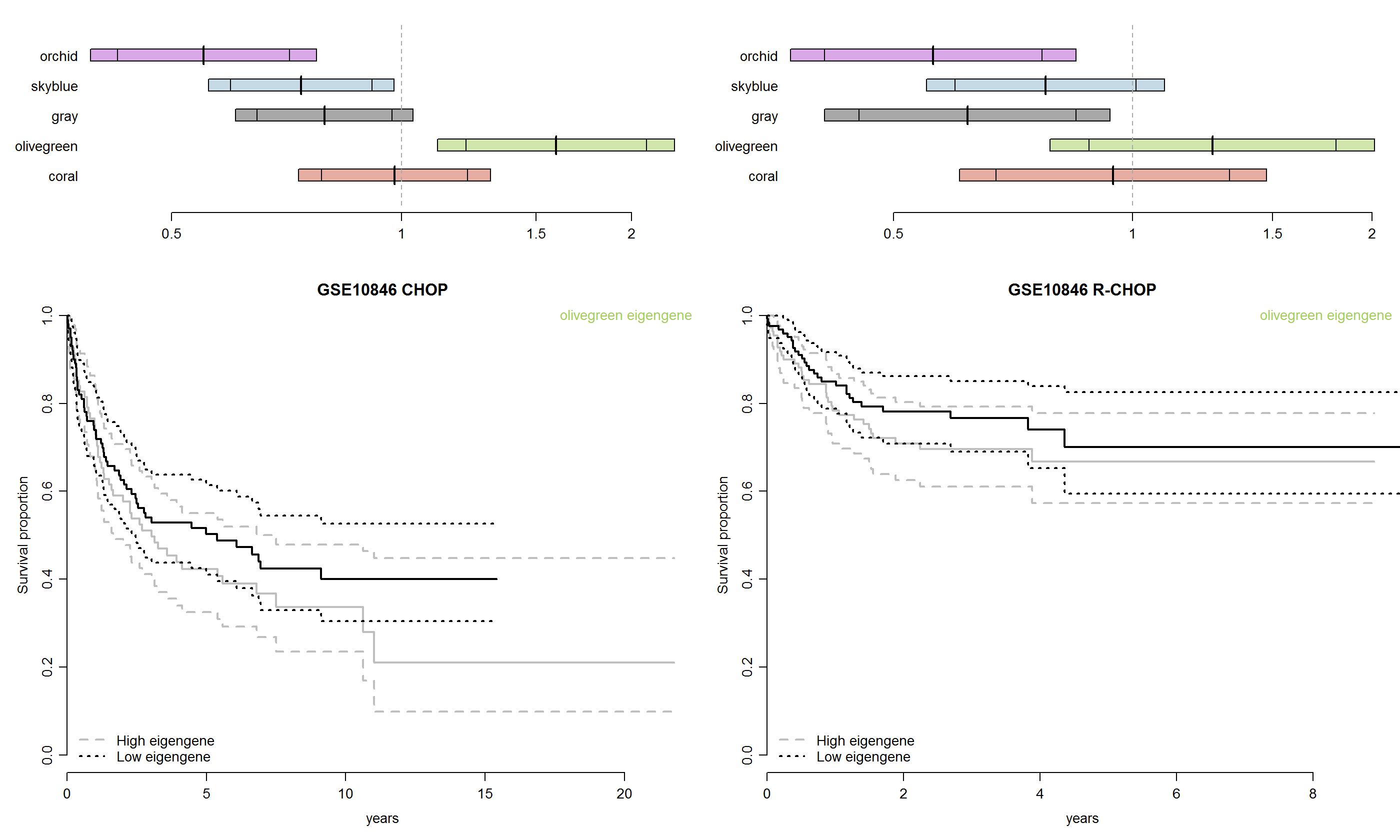}
  \caption{The top row shows $95\%$ and $99\%$ CI for the hazard ratio for each eigengene in the multiple Cox proportional hazards model containing all eigengenes in the CHOP and R-CHOP dataset. The bottom row shows Kaplan-Meier estimates (and $95\%$ CI) for the overall survival for patients stratified by the dichotomized olivegreen eigengene.}
  \label{fig:4}
\end{figure}

Next, the modules were screened for biological relevance using GO (Gene Ontology) Biological Process, Molecular Function, and Cellular Component as well as REACTOME and KEGG pathway enrichment analysis. This was done using the g:profiler web server \citep{Reimand2016} via the accompanying R-package \texttt{gProfileR} \citep{gProfileR2016}. Since we pre-selected the top 300 genes by variance, the enrichment analysis was done using only these as the background genes. 
Top genes for each module, ranked by connectivity, are shown in Table \ref{tab:top.genes.em}, while results of the enrichment analysis for each of the modules are shown in Supplementary Table \ref{tab:enrichment.em}. 
Inspection of the enrichment analysis  and most connected genes allowed us to hypothesize that the coral module is involved in "tissue development" (strong association with GO:0009888 tissue development), the Skyblue module is involved in "metastasis" (strong association with G0:0009611-response to wounding and GO.00442060-wound healing, \citep{Opdenaker2015}), the orchid module involved in "immune regulation" (strong association with GO:0002376-immune system process), and the olivegreen module involved in "immune surveillance" (strong association with GO:0006952-defense response and GO:0045087-innate immune response).

\input{table.top.genes.em}

From the gene enrichment and survival analysis the olivegreen module appeared particularly interesting, as we notice a strong involvement of immune response and an association between high value of the eigengene expression and poor survival, which eventually could make these patients candidates for experimental immunotherapies. Several of the genes, e.g. S100A8, S100A9, CD14, and CD163 with the highest connectivity in this module have been associated to immunotherapy \citep{fulmer2008, Cheng2008, pmid28331613}. As prominent examples S100A8 (MRP8; calgranulin A) and the gene S100A9 (MRP14; calgranulin B) appear in the list.  This is interesting
as compelling research has shown that the S100 family of calciumbinding
proteins maintain immunosuppressive myeloid-derived suppressor (MDS) cells at the tumor site
\citep{fulmer2008}. Notably, in mice injected with lymphoma cells, knockout
of S100A9 resulted in greater tumor infiltration of T-cells and less accumulation
of MDS cells than that seen in wild-type mice \citep{Cheng2008}.
The knockout mice had higher rates of tumor rejection and lower tumor
size than their wild-type littermates. This result indicates that knockdown
of these proteins may improve the outcome of immunotherapy strategies in
patients with values of the eigengene of the olivegreen module.

Finally, we compared the network analyses based on the covariance matrix obtained by the EM to that obtained by the Pool methods. The upper row of Supplementary Figure \ref{fig:S5} shows the heatmap and associated network modules for the Pool method, when the dendogram is cut at 5 modules, Supplementary Figure \ref{fig:S6} shows plots for the survival analysis, and top genes and gene enrichments are given in Supplementary Tables \ref{tab:top.genes.pool} and \ref{tab:enrichment.pool}. For the Pool method, we chose for each module the same color as the module of the EM based clustering with most overlapping genes. In the lower row of Figure \ref{fig:S5} a tangleram was constructed and the cophenetic correlation was calculated. We noticed generally a great overlap between the modules, but a low cophenetic correlation. With background in the simulation we anticipate the Pool method has lower efficiency than the EM method. 

The olivegreen and coral modules seem to be so tightly regulated that they manifest themselves for both methods, which is also seen in the enrichment analyses. However, the size of the skyblue module is increased for the pool method by acquiring genes from the grey module identified by the EM method, but the overall enrichment is not changed.  For the orchid module, we notice a number of genes ending up in the grey module for the pool method. This has the consequence that the immune regulation fingerprint disappers using the pool method. Morever, if we look at the less correlated intramodular connections the noise plays a larger role leading to a less clear separation between the modules for the Pool method. This can have potential biological implications, when regulating hub genes resulting in intra module cascades of reactions.

\section{Discussion}
\label{sec:conclusion}
The RCM for meta-analysis of covariance structures was shown to be superior to simple pooling as suggested previously in the literature. The estimated covariance matrix was also capable of providing a dissimilarity measure, which was able to pinpoint alternative biologically meaningful gene correlation networks in DLBCL, which can be used to formulate new hypothesis about the role of immune therapy in DLBCL.

However, the proposed testing is computationally demanding and only feasible when $p$ is sufficiently small. This could e.g.\ be overcome by improved and faster fitting procedures or by deriving the distribution of $\hat{\nu}$ under the null hypothesis. Yet the latter is seemingly intractable as $\hat{\nu}$ is a very complex function of the data. The fact that the null-hypothesis lies on the edge of the parameter space also seems to constrain the feasibility of deriving such a distribution. One might question whether the added utility of the $\nu$ parameter provides sufficient relaxation of the covariance homogeneity. Therefore, the present work should be considered a first step in the direction of explicitly modelling the inter-study variation of covariance matrices. It is also worth noticing, that although the suggested method proved to be superior to simple pooling, it only works for small or moderate numbers of features $p$. This can partly be alleviated by combining multiple studies to yield a sufficiently large total sample size $n_\bullet$ that allows for the estimation of large covariance matrices. Turning to using $p$-values seems tempting, but one should be aware, as with all hypothesis testing, that the exact threshold of ICC (or $\nu$) needed to claim homogeneous studies is dependent on the sample size and the relevant effect size. In this respect the relevant effect size is unclear and will be problem dependent.

The moderate size of $p$ is a severe drawback as many methods have been published concerning estimation of large covariance matrices by various regularization methods \citep{Meinshausen2006, Friedman2008, vanWieringen2014}. Therefore we believe this work could be further enriched by combining the method with regularized estimation. In the future such generalizations of the model to $p \gg n_\bullet$ is extremely interesting though out of scope for this article.

In conclusion the article demonstrates an advantageous model based way of conducting meta-analaysis of covariance matrices - especially in a setting with moderate number of features compared to the dimension. One should also notice the method seems to provide a generally applicable framework making it usable in other settings where multiple features are measured and believed to share a common covariance matrix obscured by group dependent noise.

\section*{Acknowledgments}
We thank Martin Raussen, Jon Johnsen, as well as Niels Richard Hansen for their assistance on some of the mathematical proofs.
The helpful comments from Steffen Falgreen, Andreas S.\ Pedersen, and reviewers were also much appreciated.
The technical assistance from Alexander Schmitz, Julie S.\ B\o{}dker, Ann-Maria Jensen, Louise H.\ Madsen, and Helle H\o{}holt is also greatly appreciated.

\begin{supplement}
\sname{Supplement A}\label{suppA}
\stitle{Appendices}
\slink[url]{http://imstat.org/aoas/}
\sdescription{Supplementary figures, tables and proofs available online.}
\end{supplement}

\begin{supplement}
\sname{Supplement B}\label{suppB}
\stitle{Documents for reproducibility}
\slink[url]{http://github.com/AEBilgrau/RCM}
\sdescription{The documents and other needed files to perform the analyses to reproduce this article. See the \texttt{README} file herein.}
\end{supplement}


\bibliographystyle{imsart-nameyear}
\bibliography{references_manual}

\printaddresses  

\newpage
\appendix
\pagenumbering{arabic}                   
\setcounter{page}{1}
\renewcommand*{\thepage}{A\arabic{page}} 

\section{Supplementary Figures and Tables}
\counterwithin{figure}{section}
\counterwithin{table}{section}
\setcounter{figure}{0}
\setcounter{table}{0}

\begin{figure}[H]
\begin{subfigure}{0.5\textwidth}
  \centering
  \includegraphics[width=\linewidth]{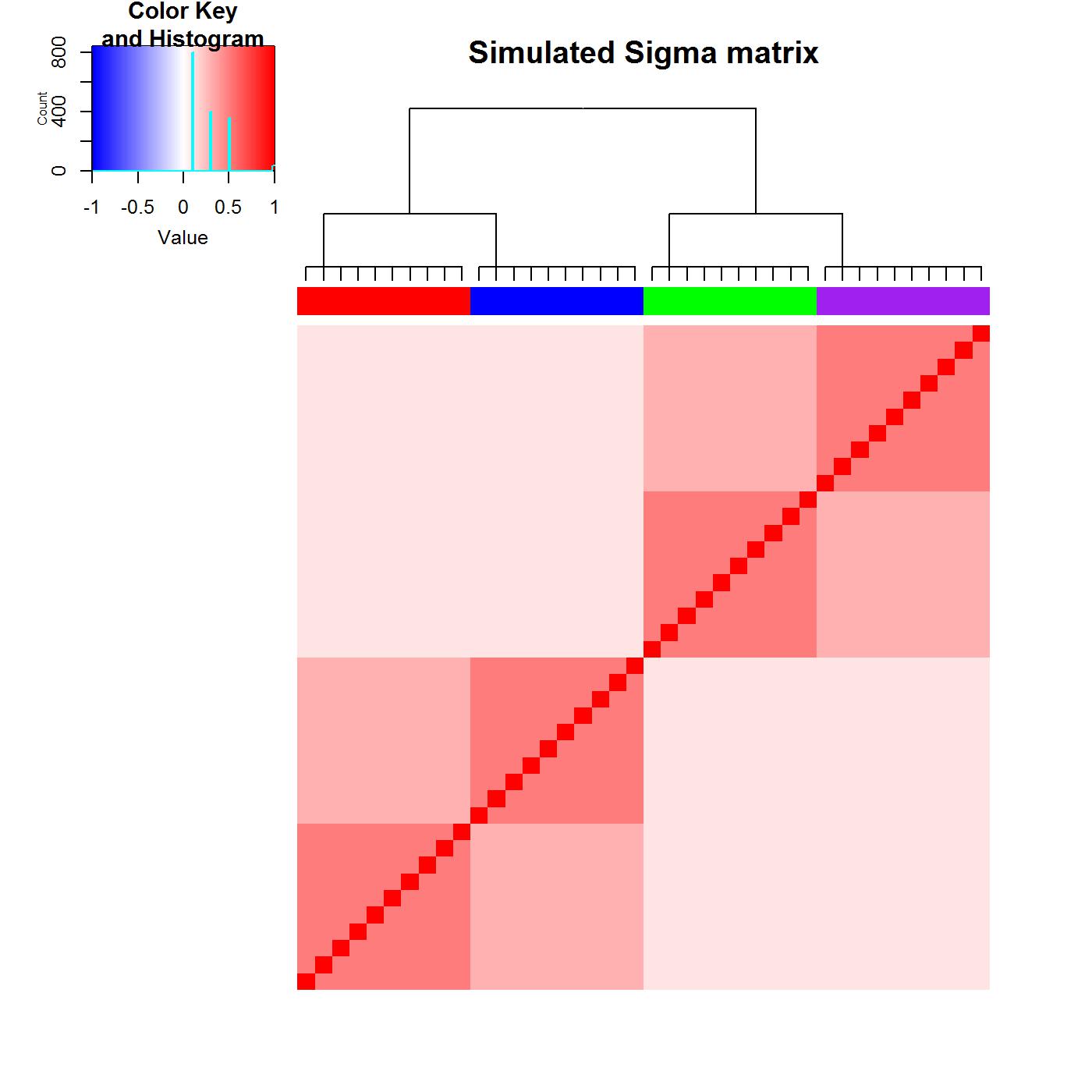}
   \label{fig:S1A}
\end{subfigure}%
\begin{subfigure}{0.5\textwidth}
  \centering
  \includegraphics[width=\linewidth]{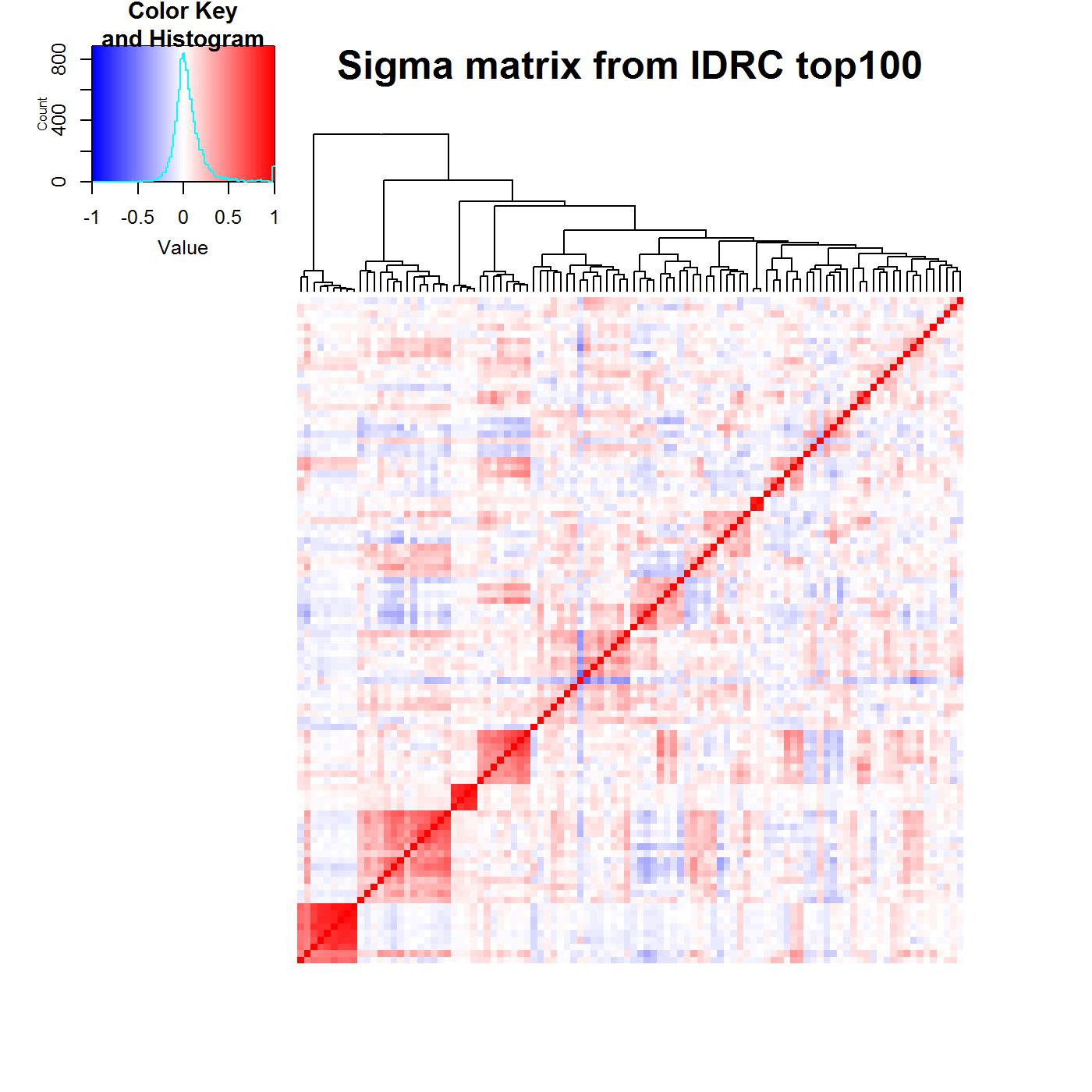}
  \label{fig:S1B}
\end{subfigure}
\caption{Heatmaps and hierarchical clustering of the $\Sigma$ matrices used for simulation.}
\label{fig:S1}
\end{figure}

\begin{figure}[H]
\begin{subfigure}{0.8\textwidth}
  \centering
  \includegraphics[width=0.8\linewidth]{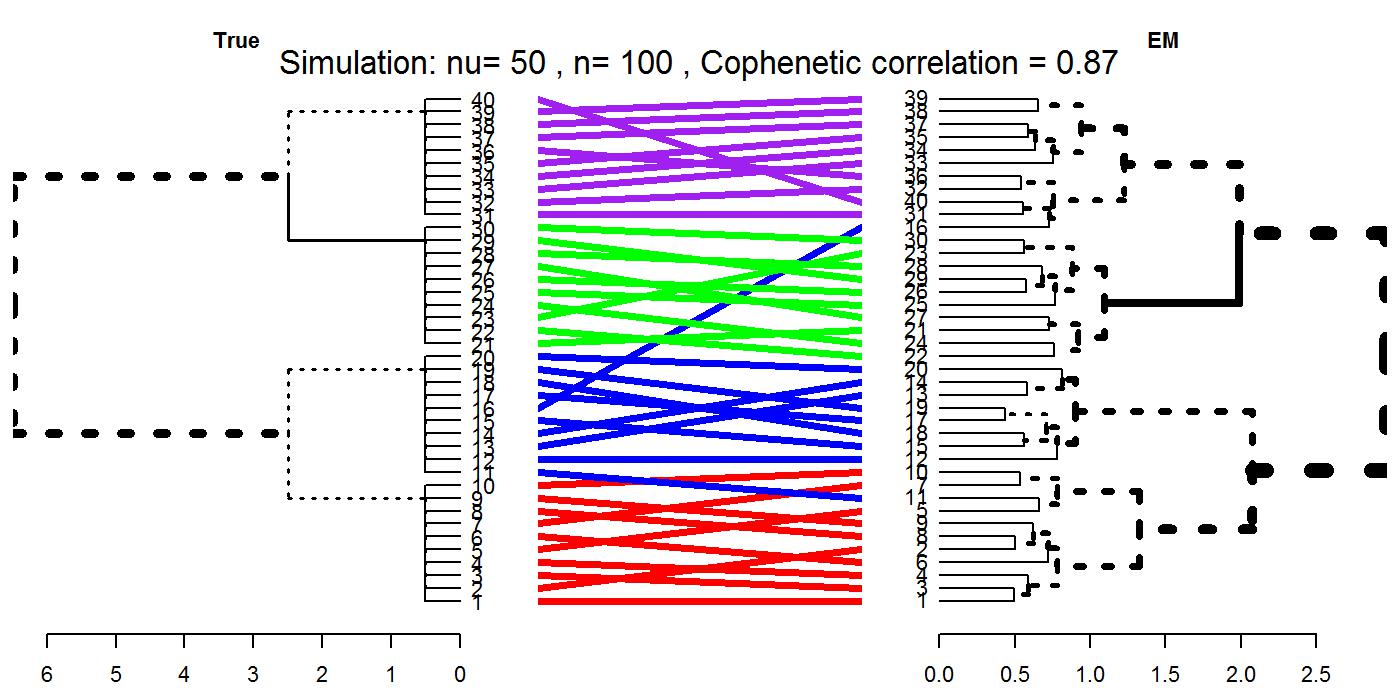}
   \label{fig:S2A}
\end{subfigure}%
\\
\begin{subfigure}{0.8\textwidth}
  \centering
  \includegraphics[width=0.8\linewidth]{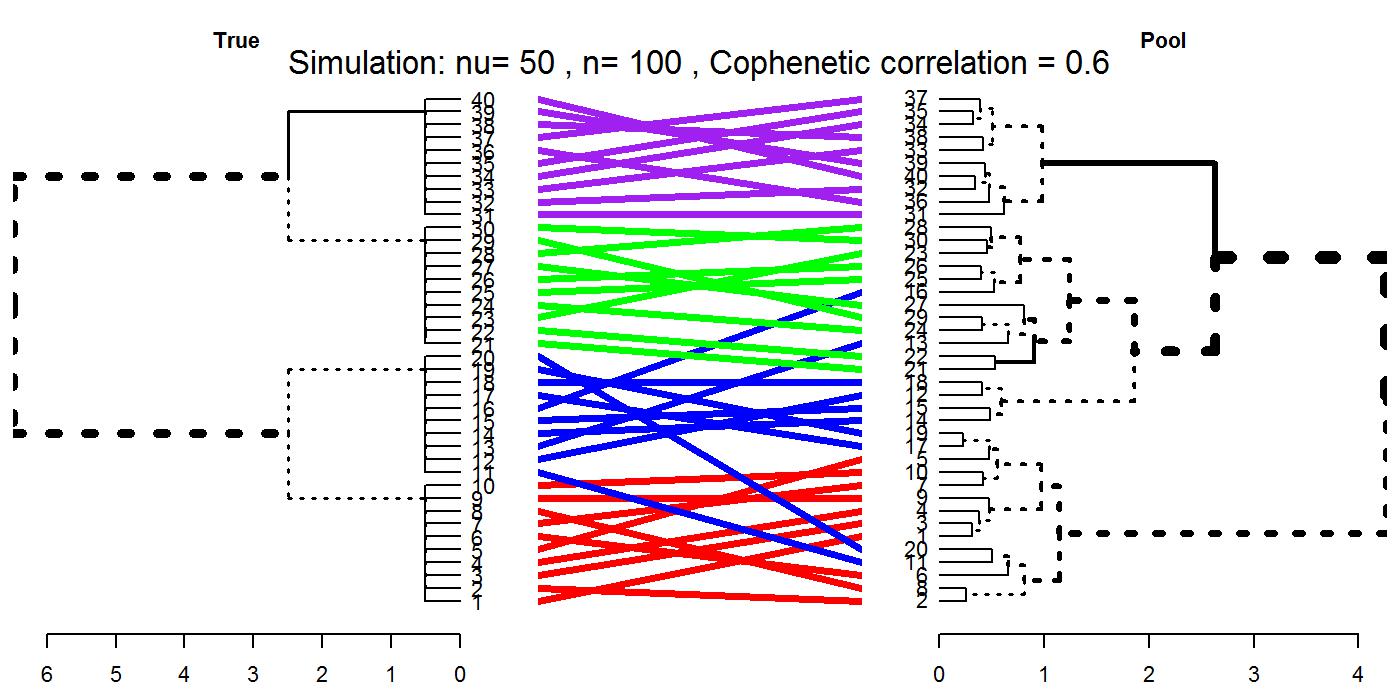}
  \label{fig:S2B}
\end{subfigure}
\caption{Tanglegrams for the True vs estimated dendrograms with the EM and Pool method.}
\label{fig:S2}
\end{figure}

\begin{figure}[H]
  \includegraphics[width=0.8\linewidth]{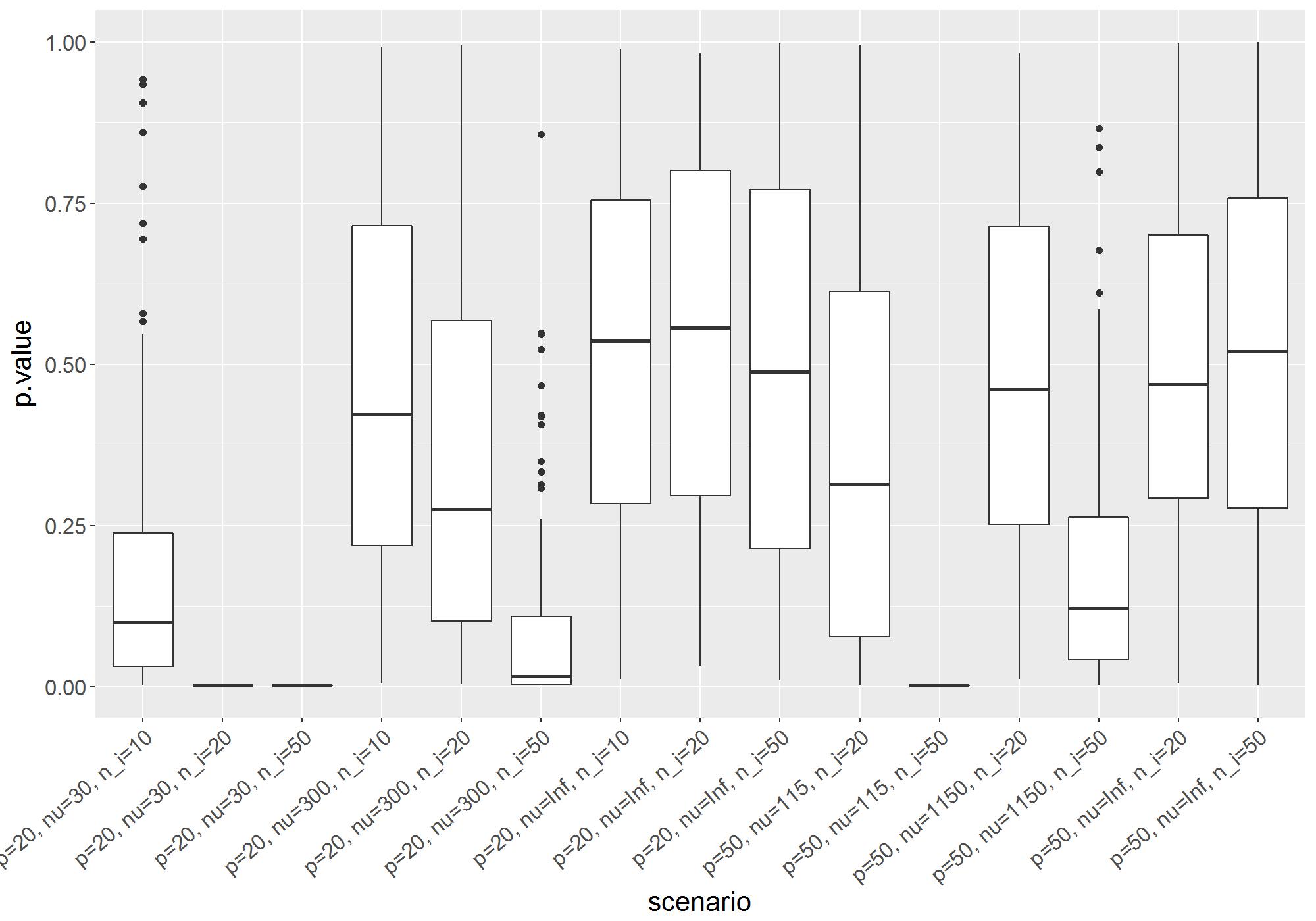}
  \caption{Boxplot of obtained P-values from the permutation procedure under different values of $p$, $\nu$ and $n_i$.}
  \label{fig:S3}
\end{figure}

\begin{figure}[H]
  \includegraphics[width=0.8\linewidth]{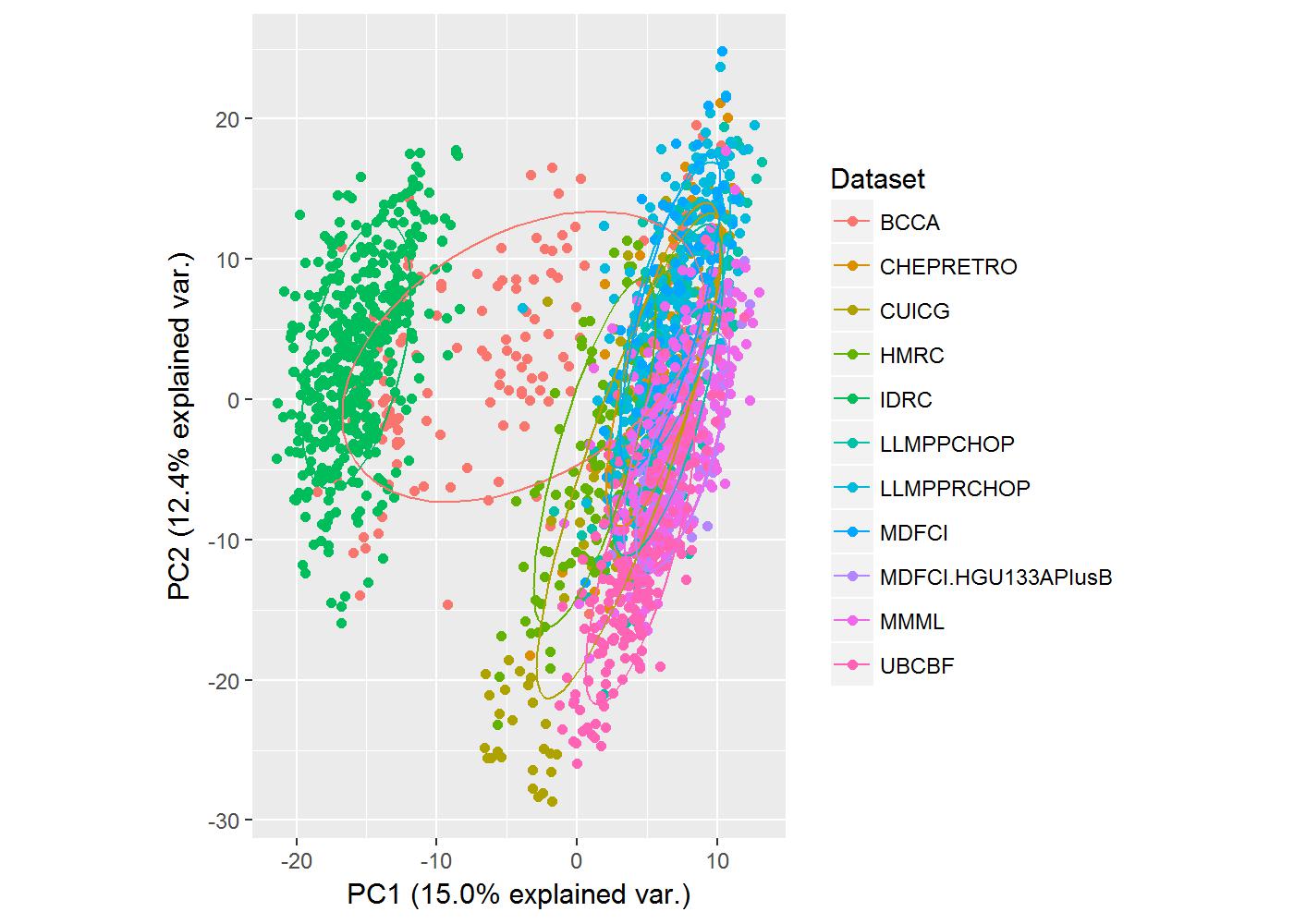}
  \caption{PCA plot of the datasets used in the DLBCL analysis.}
  \label{fig:S4}
\end{figure}

\begin{figure}[H]
\begin{subfigure}{0.5\textwidth}
  \centering
  \includegraphics[width=\linewidth]{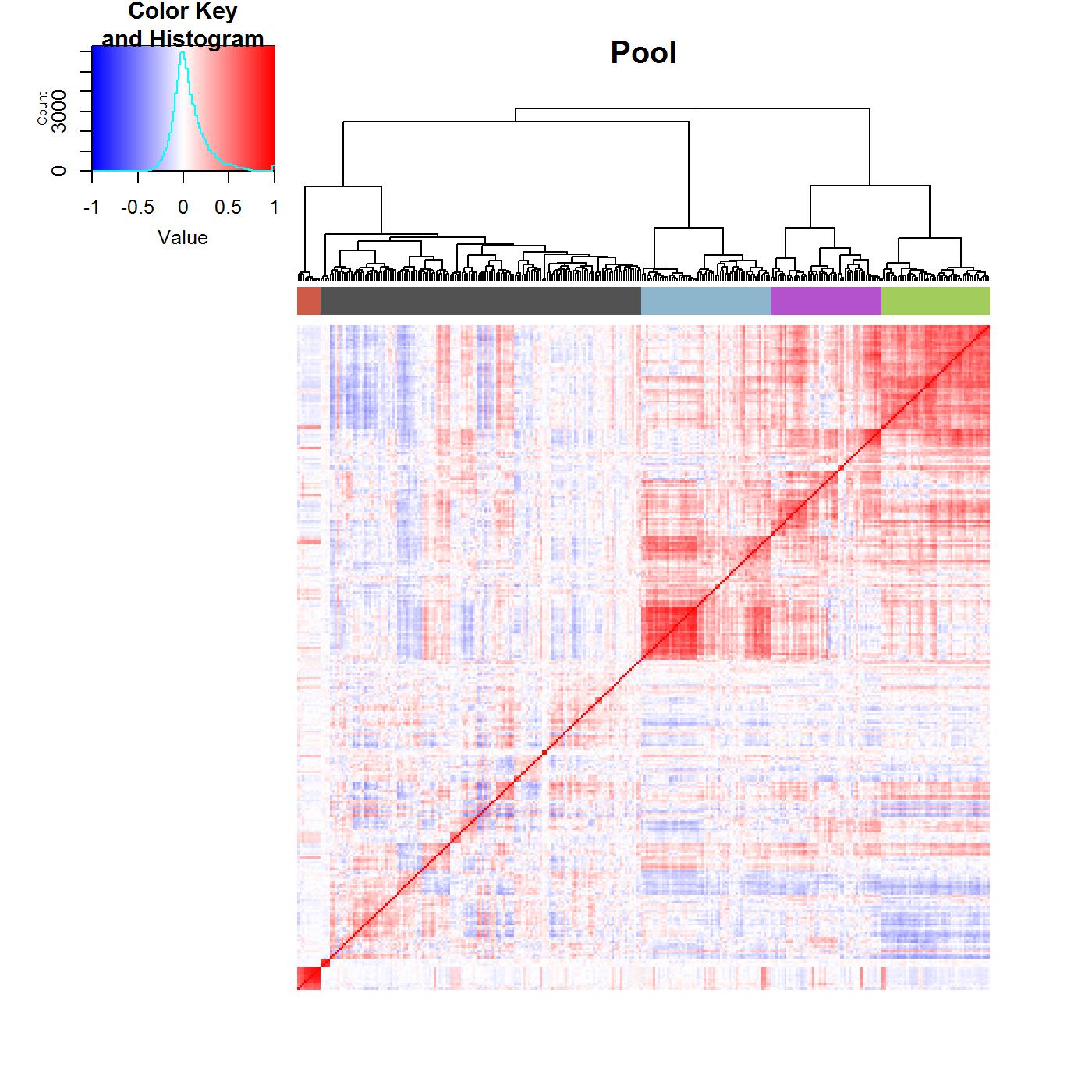}
   \end{subfigure}%
\begin{subfigure}{0.5\textwidth}
  \centering
  \includegraphics[width=\linewidth]{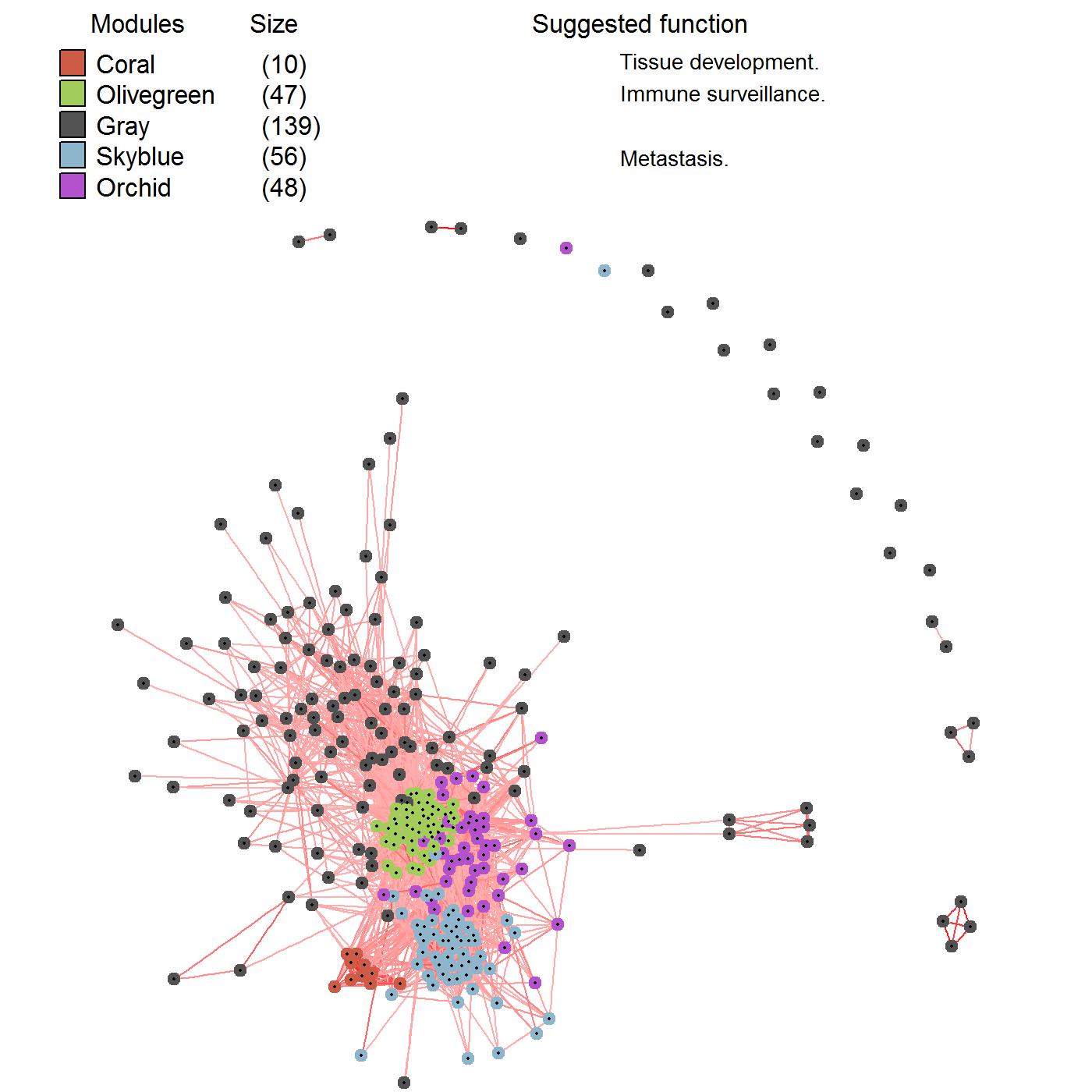}
  \end{subfigure}
\\
\begin{subfigure}{\textwidth}
	\centering
	  \includegraphics[width=\linewidth]{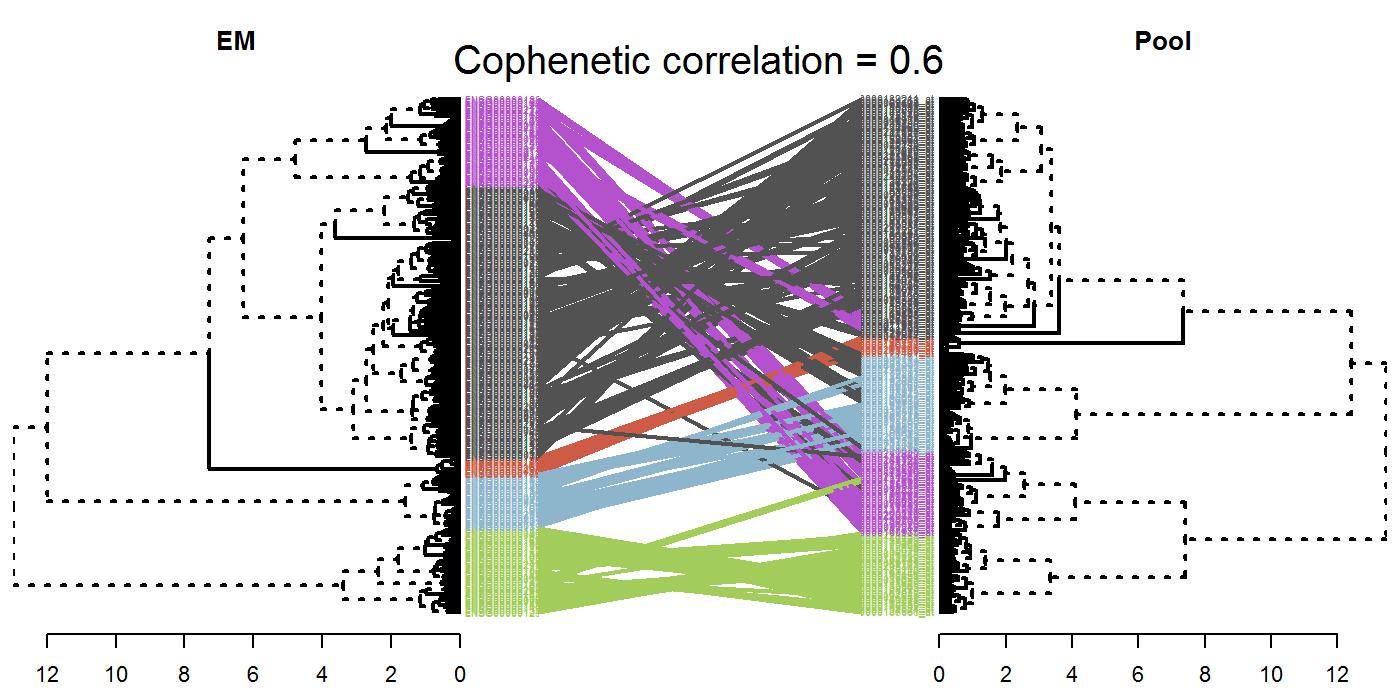}
	 \end{subfigure}
\caption{Heatmap and correlation network for the estimated correlation matrices of the top 300 genes for the DLBCL data using the Pool method. The network is cut at a height producing 5 clusters. The tanglegram in the lower panel shows the comparison of clusters between the EM and Pool methods.}
\label{fig:S5}
\end{figure}

\begin{figure}[H]
  \includegraphics[width=0.8\linewidth]{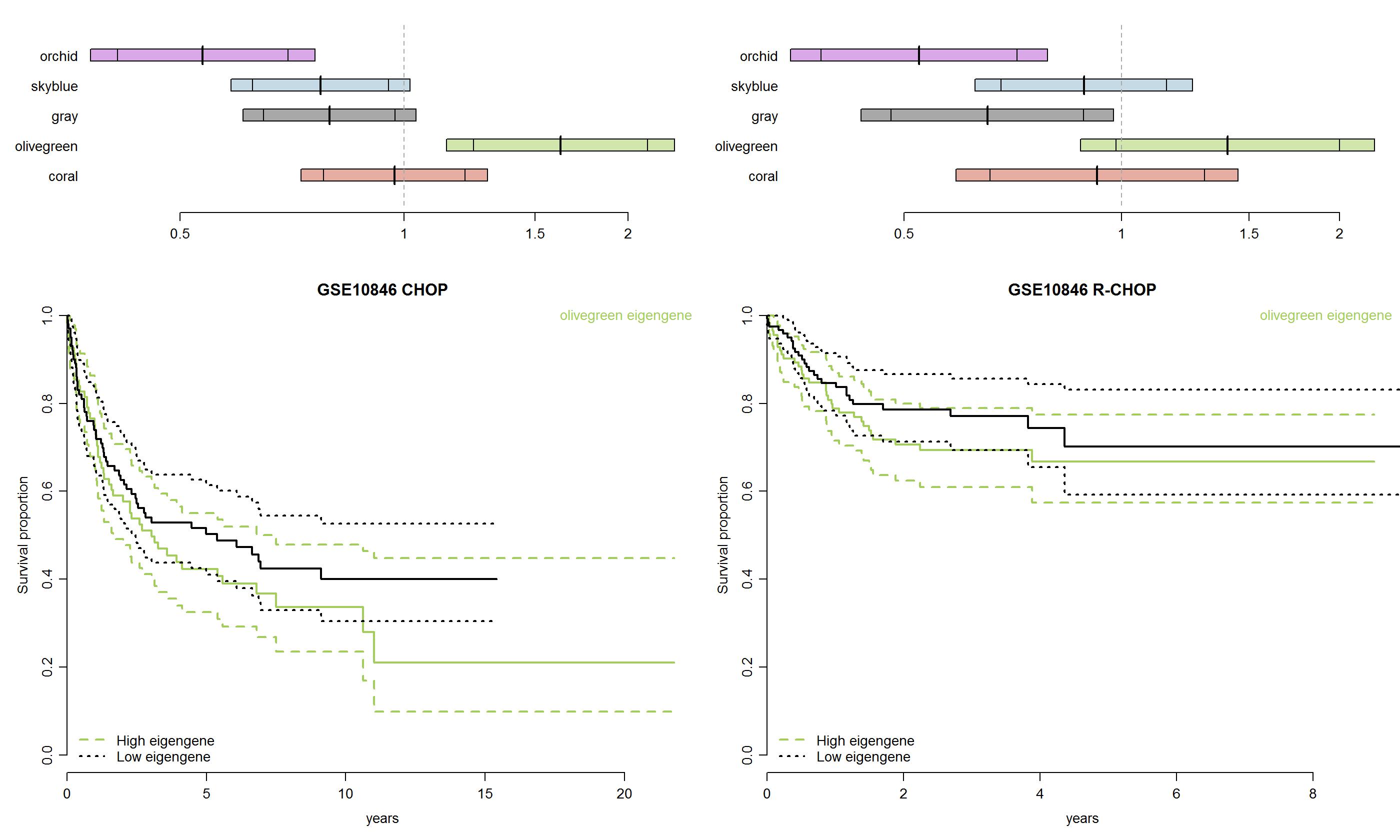}
  \caption{The top row shows $95\%$ and $99\%$ CI for the hazard ratio for each eigengene in the multiple Cox proportional hazards model containing all eigengenes in the CHOP or R-CHOP dataset. The bottom row shows Kaplan-Meier estimates (and $95\%$ CI) for the overall survival for patients stratified by the dichotomized olivegreen obtained from the Pool method. The proportion of variance explained by the eigengene in the CHOP and R-CHOP datasets for respectively the Coral, Olivegreen, Gray, Skyblue and Orchid modules were 0.72, 0.62, 0.13, 0.47, 0.33, and 0.77, 0.58, 0.13, 0.48, 0.31}
  \label{fig:S6}
\end{figure}

\begin{figure}[H]
  \includegraphics[width=0.8\linewidth]{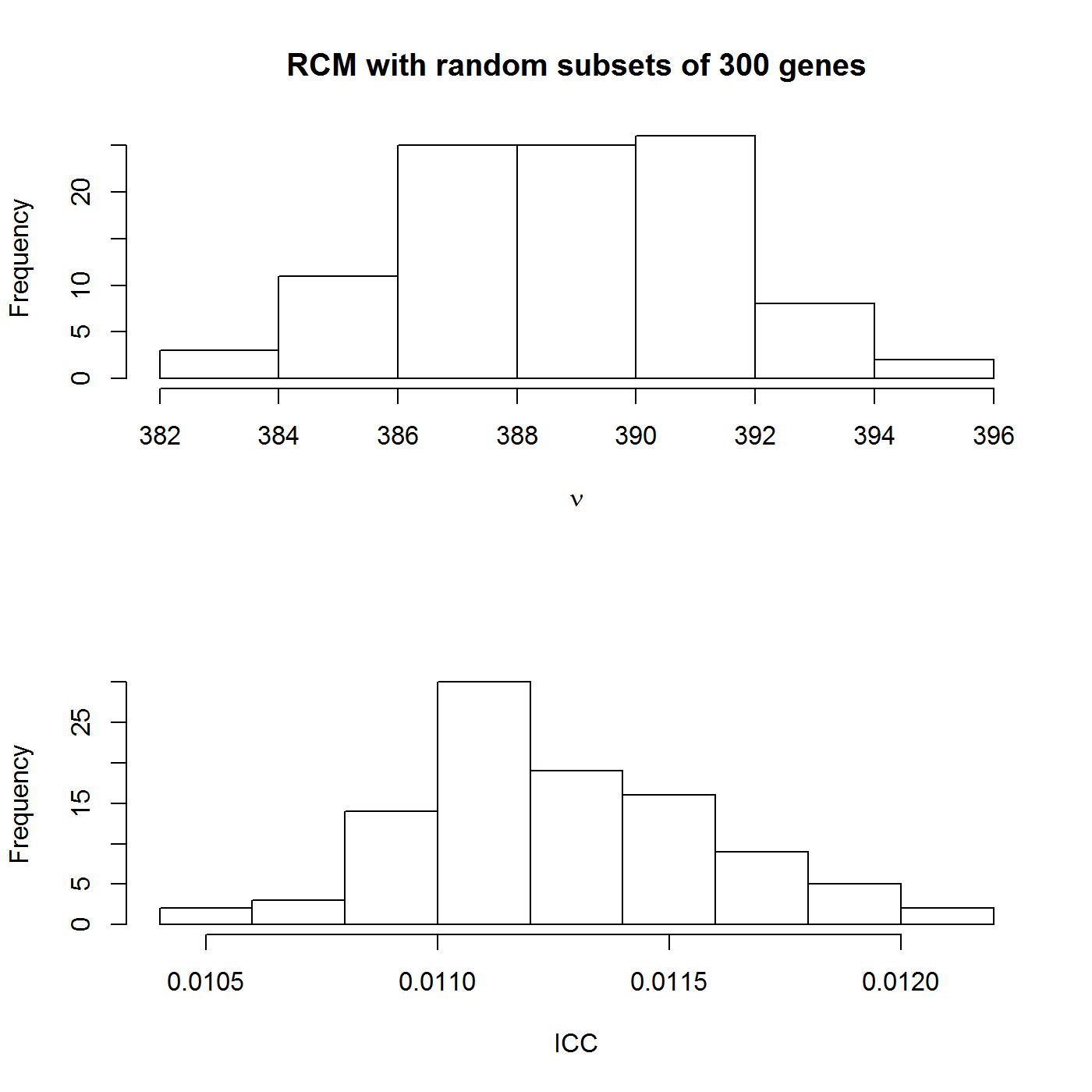}
  \caption{Distribution of $\nu$ and the corresponding ICC when fitting the RCM with the EM method to random subsets of 300 genes.}
  \label{fig:S7}
\end{figure}

\input{tableS2}
\input{tableS1}
\input{table.enrichment.em}
\input{table.top.genes.pool}
\input{table.enrichment.pool}
\newpage

\section{Marginalization of the covariance}
\label{sec:marginalization}
This section shows the marginalization over $\vSigma$ in \eqref{eq:loglik}.
Recall the model \eqref{eq:RCM} where $\calN_p(\vec{\mu},\vSigma_i)$ denotes a $p$-dimensional multivariate Gaussian distribution with mean $\vec{\mu}$ and positive definite (p.d.) covariance matrix $\vSigma_i$ with probability density function (pdf)
\begin{align}
  \label{eq:normalpdf}
  f(\vx| \vec{\mu}, \vSigma_i) =
  (2\pi)^{-\frac{p}{2}} |\vSigma_i|^{-\frac{1}{2}}
  \exp\!\left( -\frac{1}{2} (\vx - \vec{\mu})^\top \vSigma_i^{-1}(\vx - \vec{\mu}) \right),
\end{align}
and where $\calW^{-1}_p(\vPsi, \nu)$ denotes a $p$-dimensional inverse Wishart distribution with $\nu$ degrees of freedom, a p.d. $p \times p$ scale matrix $\vPsi$, and pdf
\begin{align}
  \label{eq:wishartpdf}
  f(\vSigma_i) =
  \frac{ |\vPsi|^\frac{\nu}{2} }{
        2^\frac{\nu p}{2} \Gamma_p\!\left( \frac{\nu}{2} \right) }
        |\vSigma_i|^{-\frac{\nu+p+1}{2}}
  \exp\!\left( -\frac{1}{2} \tr\!\big(\vPsi\vSigma_i^{-1}\big) \right),
  \quad\nu > p - 1,
\end{align}
where $\vSigma_i$ is p.d. and $\Gamma_p$ is the multivariate generalization of the gamma function $\Gamma$ given by
\begin{align}
  \label{eq:multigamma}
  \Gamma_p(t) = \pi^{ \frac{1}{2} \binom{p}{2} }
  \prod_{j = 1}^p \Gamma\!\left(t + \frac{1 - j}{2}\right)
  \text{ where }
  \Gamma(t) = \int_0^\infty x^{t-1} e^{-x} dx.
\end{align}

For ease of notation we drop the subscript $i$ on $\vSigma_i$, $\vX_i$, $\vS_i = \vX_i \vX_i^\top$, and $n_i$.
By the model assumptions,
{\small
\begin{align*}
  &f(\vX | \vPsi, \nu)
  = \int f(\vX|\vSigma) f(\vSigma | \vPsi, \nu) d\vSigma \\
  &= \int \left[ \prod_{j = 1}^n  (2\pi)^{-\frac{p}{2}} |\vSigma|^{-\frac{1}{2}}
                         e^{-\frac{1}{2}\tr(\vx_{ij}\vx_{ij}^\top\vSigma^{-1})} \right]
          \frac{\big|\vPsi\big|^\frac{\nu}{2}}
               {2^{\frac{\nu p}{2}}\Gamma_p\!\left(\frac{\nu}{2}\right)}
          |\vSigma|^{-\frac{\nu+p+1}{2}}e^{-\frac{1}{2}\tr\!\big(\vPsi\vSigma^{-1}\big)}
      \;d\vSigma \\
  &= (2\pi)^{-\frac{np}{2}}
      \frac{\big|\vPsi\big|^\frac{\nu}{2}}
           {2^{\frac{\nu p}{2}}\Gamma_p\!\left(\frac{\nu}{2}\right)}
      \int
        |\vSigma|^{-\frac{n}{2}}  e^{-\frac{1}{2}\tr(\vS\vSigma^{-1})}
        |\vSigma|^{-\frac{\nu+p+1}{2}} e^{-\frac{1}{2}\tr\!\big(\vPsi\vSigma^{-1}\big)}
      \;d\vSigma \\
  &=
      \frac{\big|\vPsi\big|^\frac{\nu}{2}}
           {\pi^\frac{np}{2} 2^{\frac{(\nu + n) p}{2}}\Gamma_p\!\left(\frac{\nu}{2}\right)}
      \int
        |\vSigma|^{-\frac{(\nu + n)+p+1}{2}}
         e^{-\frac{1}{2}\tr\!\Big(\big(\vPsi+ \vS\big)\vSigma^{-1}\Big)}
      \;d\vSigma.
\end{align*}
\normalsize}%
The integrand can be recognized as a unnormalized inverse Wishart pdf of the distribution $\calW^{-1}\big(\vPsi + \vS, \nu + n\big)$, and so the integral evaluates to the reciprocal value of the normalizing constant in that density. Thus,
{\small
\begin{align}
  f(\vX | \vPsi, \nu)
  =
    \frac{\big|\vPsi\big|^\frac{\nu}{2}}
         {\pi^\frac{np}{2} 2^{\frac{(\nu + n) p}{2} } \Gamma_p\!\left(\frac{\nu}{2}\right)}
    \frac{2^\frac{(v+n)p}{2} \Gamma_p\left(\frac{\nu + n}{2}\right)}
         {\big|\vPsi + \vS\big|^{\frac{\nu + n}{2}}} 
  =
    \frac{\big|\vPsi\big|^\frac{\nu}{2} \Gamma_p\left(\frac{\nu + n}{2}\right)}
         {\pi^\frac{np}{2}
           \big|\vPsi + \vS\big|^{\frac{\nu + n}{2}} \Gamma_p\!\left(\frac{\nu}{2}\right)}.
    \label{eq:marg1}
\end{align}
\normalsize}%
Using the matrix determinant lemma and $\vS = \vX^\top\vX$, this can be further simplified to
\begin{align*}
  f(\vX | \vPsi, \nu)
  =
  \frac{\Gamma_p\left(\frac{\nu + n}{2}\right)}
       {\pi^\frac{np}{2}
          \big| \vI + \vX\vPsi^{-1}\vX^\top\big|^\frac{\nu + n}{2}
          \big|\vPsi\big|^\frac{n}{2}
          \Gamma_p\!\left(\frac{\nu}{2}\right)},
\end{align*}
which can help to speed-up computations.

\section{Proofs}\label{sec:proofs}

\subsection{Non-concavity of the log-likelihood}\label{sec:concaveloglik}
The likelihood function is not log-concave in general.
This section analyses the (non)-concavity of the log-likelihood function given in \eqref{eq:loglik}.
More precisely, the following two propositions are proved.

\propositionNonConcavityInPsi*

\propositionConcavityInNu*

\begin{proof}[\textbf{Proof of Proposition \ref{prop:nonconcavityinpsi}}]
Assume $\nu$ is fixed and consider only the terms involving $\vPsi$ in \eqref{eq:loglik}.
We reduce to the one-dimensional case where
\begin{align*}
  \ell(\psi)
  = \frac{k\nu}{2}\log\!\big(\psi\big)
     - \sum_{i = 1}^k \frac{\nu + n_i}{2}\log\!\big(\psi + x_i^2\big),
\end{align*}
which implies
{\small
\begin{align*}
  \ell'(\psi)
  = \frac{k\nu}{2}\frac{1}{\psi}
     - \sum_{i = 1}^k \frac{\nu + n_i}{2}\frac{1}{\psi + x_i^2}
  \text{ and }
  \ell''(\psi)
  =  - \frac{k\nu}{2}\frac{1}{\psi^2}
      + \sum_{i = 1}^k \frac{\nu + n_i}{2}\frac{1}{\big(\psi + x_i^2\big)^2}.
\end{align*}
\normalsize}%
It is straightforward to show there exists a value for $\psi$, $n_i$ and $\nu$ for which $\ell''(\psi) > 0$.
Since the second derivative is not always negative the log-likelihood $\ell$ is not log-concave.
\end{proof}

\begin{proof}[\textbf{Proof of Proposition \ref{prop:concavityinnu}}]
Consider the terms involving $\nu$.
Clearly, the mixed terms involving both $\nu$ and $\vPsi$ are log-linear in $\nu$ and hence log-concave.
We thus restrict our attention to the remaining terms not dependent on $\vPsi$.
The sum of these terms are concave in $\nu$, since
\begin{align*}
  &\log\Gamma_p\!\left( \frac{\nu + n_i}{2} \right) -
    \log\Gamma_p\!\left( \frac{\nu}{2} \right)
  =  \log\frac{\Gamma_p\!\left( \frac{\nu + n_i}{2} \right)}{
               \Gamma_p\!\left( \frac{\nu}{2}       \right)}
  = \sum_{j = 1}^p \log
    \frac{\Gamma\!\big( \frac{\nu + 1 - j}{2} + \frac{n_i}{2} \big)}{
          \Gamma\!\big( \frac{\nu + 1 - j}{2} \big)}.
\end{align*}
which can be seen to be concave since $n_i \geq 1$ for all $i$ and
$
  h(x) = \log\!\big(\frac{\Gamma(x + a)}{\Gamma(x)}\big)
$
is concave for all $x>0$ and $a > 0$.
The concavity of $h$ is easily seen by the fact that
$
  h''(x) = \psi(x + a) - \psi(x) < 0,
$
where $\psi(\cdot)$ is the tri-gamma function.
The tri-gamma function is a well-known monotonically decreasing function.
Hence, the likelihood is log-concave in $\nu$.
\end{proof}

\subsection{Existence and uniqueness of likelihood maxima}
\label{sec:negativedefinite}
This section proves Lemmas \ref{lem:elltominusinfty} and \ref{lem:negativesdefinite} which imply Proposition \ref{prop:uniquemax}.

Before we state the lemmas, the proposition, and their proofs, we see that the reparameterisation of the RCM is irrelevant.
Consider the log-likelihood in \eqref{eq:loglik} assuming $\nu$ fixed.
The log-likelihood obey
\begin{align}
  \label{eq:loglik2}
  2\ell(\vPsi)
  &= c + k\nu\log\big|\vPsi\big| - \sum_{a=1}^k (n_a + \nu)\log\big|\vPsi + \vS_a\big|.
\end{align}
Notice, that this equation also holds in the reparameterization. Here we have
\begin{align*}
  2\ell(\vSigma)
  &= c + k\nu\log\big|(\nu-p-1)\vSigma\big| - \sum_{a=1}^k (n_a + \nu)\log\big|(\nu-p-1)\vSigma + \vS_a\big|\\
  &= c' + k\nu\log\big|\vSigma\big| - \sum_{a=1}^k (n_a + \nu)\log\big|\vSigma + (\nu-p-1)^{-1}\vS_a\big|.
\end{align*}
Since $(\nu-p-1)^{-1}\vS_a$ is only dependent on data (when $\nu$ is fixed) we can set $(\nu-p-1)^{-1}\vS_a := \vS_a$.
Without loss of generality we can therefore consider \eqref{eq:loglik2} in the following.

\propositionUniqueMax*

\begin{proof}[\textbf{Proof of Proposition \ref{prop:uniquemax}}]
We first prove existence of the maximum.
Note, that we may consider $\ell$ as a function on a vector space by letting $\vPsi = \exp(\vX)$ where $\vX$ is a symmetric matrix.
By Lemma \ref{lem:elltominusinfty} and the continuity of $\ell$, the set
$
  \big\{  \vPsi \big| \ell(\vPsi) \geq \ell(\vPsi^*)  \big\}
$
is bounded and closed and thus compact for any $\vPsi^*\succ 0$.
The existence of a maximum follows from the extreme value theorem by the continuity of $\ell$.
A stationary point exists due to Rolle's theorem and the differentiability of $\ell$.

Next, we show the uniqueness of the maximum.
Let $\cal(ST)$ denote the set of stationary points, which is nonempty.
By Lemma \ref{lem:elltominusinfty}, $\ell(\vPsi)$ has a finite upper bound given by the maximum of the log-likelihood in those points.
All gradient curves (that is, solution curves to $\dot{\vPsi}(t) = \nabla \ell(\vPsi(t))$) must then converge toward exactly one of the stationary points where $\ell$ monotonically increases along each curve.
Define for $\vPsi_s in \cal{ST}$ the basin of attraction
\begin{align*}
   A_s = \big\{
     \vPsi_0 \in \calS_+ \big|
     \vPsi(0) = \vPsi_0, \;
     \lim_{t \to \infty} \vPsi(t) = \vPsi_s
  \big\},
\end{align*}
The basin of attraction is open if $\vPsi_s$ is a maximum \citep[Lemma 4.1]{Khalil2002}.
By Lemma \ref{lem:negativesdefinite}, $\vPsi_s$ is always a maximum and hence all $A_s$ are open sets in the set of all positive definite matrices $\calS_+$.
This partitions the space $\calS_+$ into disjoint, non-empty, open sets.
Since $\calS_+$ is connected, this is only possible if $A_s = \calS_+$ and thus there is only a single basin of attraction and maximum of $\ell$.
\end{proof}

\begin{restatable}{lemma}{lemmaOne}
\label{lem:elltominusinfty}
If there exists an eigenvalue $\lambda_t$ of $\vPsi_t$ such that $\lambda_t \to 0$ or   $\lambda_t \to \infty$, then $\ell(\vPsi_t) \to -\infty$ for $\nu$ fixed and $n_\bullet = \sum_{a=1}^k n_a \geq p$.
\end{restatable}

\begin{proof}[\textbf{Proof of Lemma \ref{lem:elltominusinfty}}]
Assume the hypothesis of the lemma and consider the expression given in \eqref{eq:loglik2} up to the addition of a constant.
The likelihood obey the following two upper bounds.
First,
\begin{align*}
  \ell(\vPsi_t)
  &= \frac{k\nu}{2}\log\big|\vPsi_t\big| - \sum_{i = 1}^k \frac{\nu + n_i}{2} \log |\vPsi_t+\vS_i| \\
  &\leq \frac{k\nu}{2}\log|\vPsi_t| - \sum_{i = 1}^k \frac{\nu + n_i}{2} \log |\vPsi_t|
  =  - \frac{n_\bullet}{2} \log |\vPsi_t|
\end{align*}
Secondly, let
$
  C = \sum_{i = 1}^k \frac{\nu + n_i}{2} = \frac{k\nu}{2} + \frac{n_\bullet}{2},
$
whereby \eqref{eq:loglik2} can be expressed as
\begin{align*}
  \ell(\vPsi_t)
  = \frac{k\nu}{2}\log|\vPsi_t| -
    C \sum_{i = 1}^k \frac{\nu + n_i}{2C} \log |\vPsi_t+\vS_i|.
\end{align*}
Since $\log|\cdot|$ is concave and the above sum is a convex combination, we have
\begin{align*}
  \ell(\vPsi_t)
  \leq \frac{k\nu}{2}\log|\vPsi_t| -
     C \log\left| \vPsi_t + \sum_{i = 1}^k \frac{\nu + n_i}{2C}\vS_i\right|.
\end{align*}
Hence,
\begin{align*}
  \ell(\vPsi_t)
  \leq \min\Bigl\{
    - \frac{n_\bullet}{2} a(t) , \;
      \frac{k\nu}{2} a(t)  - C \log\left| \vPsi_t + \vS\right|
  \Bigr\}
\end{align*}
where $a(t) = \log|\vPsi_t|$ and $\vS = \sum_i \frac{\nu+n_i}{2C}\vS_i$.
Three cases now exists:
1) If $a(t) \to \infty$, then
\begin{align*}
  \ell(\vPsi_t)
  \leq - \frac{n_\bullet}{2} a(t) \to -\infty.
\end{align*}
2) If $a(t) \to -\infty$, then
\begin{align*}
  \ell(\vPsi_t)
  \leq \frac{k\nu}{2} a(t) - C \log\left| \vPsi_t + \vS\right|
  \leq \frac{k\nu}{2} a(t) - C \log\left| \vS \right| \to -\infty
\end{align*}
as the matrix in the second term is almost surely positive definite when $n_\bullet = \sum_{i=1}^k n_a \geq p$ and the log determinant is some constant.
3) If $a(t)$ is bounded and the largest eigenvalue $\lambda_\text{max}(\vPsi_t) \to \infty$ (and hence $\lambda_\text{min}(\vPsi_t)  \to -\infty)$, then
$\lambda_\text{max}(\vPsi_t+\vS) \to \infty$
and
$\lambda_\text{min}(\vPsi_t+\vS)$ is bounded away from zero.
Therefore,
\begin{align*}
  \ell(\vPsi_t) \leq \frac{k\nu}{2}a(t) - C\log|\vPsi_t + \vS| \to -\infty,
\end{align*}
which completes the proof.
\end{proof}

\begin{restatable}{lemma}{lemmaTwo}
\label{lem:negativesdefinite}
If $n_\bullet \geq p$ and $\nu$ is fixed then the Hessian of the log-likelihood \eqref{eq:loglik} is negative definite in all stationary points.
\end{restatable}

\begin{proof}[\textbf{Proof of Lemma \ref{lem:negativesdefinite}}]
We show the conclusion of the Lemma directly by differentiation of $\ell$ w.r.t.\ $\vPsi$.
To do so, the matrix cookbook by \citet{Petersen2008} is a useful reference.
In particular, see equations (41, p.\ 8) and (59, p.\ 9) and pages 14 and 52--53.
We first compute expressions for the first and second order derivatives.

\textbf{First order derivatives.}
From the log-likelihood expression, we compute the first order derivative $\nabla_\vPsi 2\ell(\vPsi)$ which is the matrix-valued function where each entry is given by
\begin{align}
  \frac{\partial 2\ell}{\partial \Psi_{ij}}
  = k\nu\tr\!\left(\vE^{ij}\vPsi^{-1}\right)
    - \sum_{a = 1}^k (\nu + n_a)\tr\!\left(\vE^{ij}\left(\vPsi + \vS_a\right)^{-1}\right).
\label{eq:dloglik}
\end{align}
and $\vE^{ij}$ is a matrix with ones at entries $(i,j)$ and $(j,i)$ and zeros elsewhere.
This $\vE^{ij}$ is introduced as the derivative is not straight-forward because of the symmetric structure of $\vPsi$. Had $\vPsi$ been unstructured, then $\frac{\partial}{\partial \vPsi}\log|\vPsi| = \vPsi^{-1}$.
However, when $\vPsi$ is symmetric we have that $\frac{\partial}{\partial \Psi_{ij}}\log|\vPsi| = \tr(\vE^{ij}\vPsi^{-1})$ which is the same as $\frac{\partial}{\partial \vPsi}\log|\vPsi| = 2\vPsi^{-1} -\vPsi^{-1} \circ \vI$ where $\circ$ denotes the Hadamard product \citep[eq.\ (43) and (141)]{Petersen2008}.

The first order derivative lives in a $\binom{p+1}{2}$-dimensional vector space with basis vectors $\vE^{ij}$ indexed by $(i,j)$, $i\leq j$.

\textbf{Second order derivatives.}
We proceed with the second order derivative $\nabla^2_\vPsi 2\ell(\vPsi)$ with entries given by
\begin{align*}
  \frac{\partial^2 2\ell}{\partial \Psi_{kl} \partial \Psi_{ij}}
  &= - k\nu\tr\!\left( \vE^{ij}\vPsi^{-1} \vE^{kl}\vPsi^{-1} \right) \\
  & + \sum_{a = 1}^k (\nu + n_a)
    \tr\!\left(
      \vE^{ij}\left(\vPsi + \vS_a\right)^{-1}
      \vE^{kl}\left(\vPsi + \vS_a\right)^{-1}
    \right),
\end{align*}
obtained by differentiation of \eqref{eq:dloglik} using
$\frac{\partial}{\partial \Psi_{ij}} \vPsi^{-1} = - \vPsi^{-1}\vE^{ij}\vPsi^{-1}$ \citep[eq.\ (40)]{Petersen2008} and the linearity of the trace operator.

The second order derivative is a $\binom{p+1}{2} \times \binom{p+1}{2}$-dimensional matrix indexed by $(i,j)$ and $(k,l)$, $i \leq j$, $k \leq l$.

\textbf{Negative definiteness of stationary points.}
With the above expressions we now show that the Hessian matrix is negative definite in all stationary points.
Let $\vY = \sum_{(i,j)} y_{ij}\vE^{ij}$ be an arbitrary symmetric matrix in the vector space where $\vY \neq \vec{0}$.
In our vector space we need to show that
\begin{align*}
  \sum_{i\leq j, k\leq l}
    Y_{ij}
    \left(\nabla^2_\vPsi 2\ell(\vPsi)\right)_{(i,j),(k,l)}
    Y_{kl}
    < 0
\end{align*}
holds in every stationary point analogous to $\vec{z}^\top \vec{A}\vec{z} = \sum_{ij} A_{ij} z_i z_j < 0$.
From the second derivative, this amounts to showing that in every stationary point,
{\small
\begin{align}
- k\nu\tr\!\left( \vY\vPsi^{-1} \vY\vPsi^{-1} \right)
+ \sum_{a = 1}^k (\nu + n_a)
    \tr\!\left(
      \vY\left(\vPsi + \vS_a\right)^{-1}
      \vY\left(\vPsi + \vS_a\right)^{-1}
    \right)
    < 0.
  \label{eq:negativedefinte}
\end{align}
\normalsize}%
Now, by the positive-definiteness of $\vPsi$, let
\begin{align*}
  \vY &:= \vPsi^{-\frac{1}{2}} \vY \vPsi^{-\frac{1}{2}} \text{ and } \\
  \vS_a &:= \vPsi^{-\frac{1}{2}} \vS_a  \vPsi^{-\frac{1}{2}},
\end{align*}
and thus without loss of generality we can assume that $\vPsi = \vI$.
Hence, the derivative of the likelihood \eqref{eq:dloglik} equated to zero, becomes
\begin{align*}
  k\nu\vI = \sum_a(n_a + \nu)(\vI + \vS_a)^{-1}
\end{align*}
which implies (by multiplication by $\vY$ on each side) that every stationary point obey
\begin{align}
  k\nu\tr(\vY^2)
  &= \sum_a (n_a + \nu)\tr\!\Big(\vY(\vI + \vS_a)^{-1}\vY\Big).
  \label{eq:loglikequation}
\end{align}
We substitute \eqref{eq:loglikequation} into \eqref{eq:negativedefinte} to get
\begin{align*}
  &\sum_a (n_a + \nu)
  \tr\!\Big(\vY (\vI + \vS_a)^{-1} \vY (\vI + \vS_a)^{-1} - \vY (\vI + \vS_a)^{-1} \vY \Big) \\
  &=  \sum_a (n_a + \nu)
  \tr\!\Big( \vY (\vI + \vS_a)^{-1} \vY \big[ (\vI + \vS_a)^{-1}  - \vI \big]\Big)
  < 0.
\end{align*}
We note that $\vS_a = \vX_a\vX_a^\top$ and
\begin{align*}
  (\vI + \vS_a)^{-1} - \vI = -\vX_a\big(\vI + \vX_a^\top\vX_a\big)^{-1}\vX_a^\top,
\end{align*}
by the matrix inversion lemma whereby we need to show that
\begin{align*}
  \sum_a (n_a + \nu) \tr\!\Big(
    \vY(\vI + \vX_a\vX_a^\top)^{-1}\vY\vX_a\big(\vI + \vX_a^\top\vX_a\big)^{-1}\vX_a^\top
  \Big)
  > 0.
\end{align*}
Assume that the sum is actually zero.
Since $(\vI + \vX_a\vX_a^\top)^{-1}\succ 0$ we then obtain that
\begin{align*}
  \vY\vX_a(\vI + \vX_a\vX_a^\top)^{-1}\vX_a^\top \vY = 0
  \quad\text{ for } a = 1, ...., k.
\end{align*}
Again by $(\vI + \vX_a\vX_a^\top)^{-1}\succ 0$ we conclude that $\vY\vX_a = 0$ for all $a = 1,...,k$, i.e.\
$\vY(\vX_1, ..., \vX_k) = 0$. If $n_\bullet\geq p$ then almost surely $(\vX_1, ..., \vX_k)$ has rank $p$ whereby $\vY=0$.
\end{proof}

\section{Likelihood of the precision matrix}
\label{sec:precisionloglik}
Suppose we have $k$ i.i.d. realizations, $\vDelta_1, ..., \vDelta_k$, from the Wishart distribution given in model \eqref{eq:precisiondensity}.
The corresponding log-likelihood can be computed straight-forwardly:
\begin{align*}
  \ell(\vTheta | \vDelta_1, ..., \vDelta_k)
  &= \sum_{i = 1}^k \log
    \frac{\big|\vTheta\big|^{-\frac{\nu}{2}}}
         {2^{-\frac{vp}{2}}\Gamma_p\!\left(\frac{\nu}{2}\right)}
    |\vDelta_i|^\frac{\nu - p - 1}{2}e^{-\frac{1}{2}\tr\!\big(\vTheta^{-1}\vDelta_i\big)}\\
   &= c + \sum_{i = 1}^k \left(
     -\frac{\nu}{2} \log \big|\vTheta\big|
     -\frac{1}{2}\tr\!\big(\vTheta^{-1}\vDelta_i\big)
     \right)\\
   &= c - \frac{\nu k}{2}
     \left(
       \log |\vTheta| +
       \tr\!\left(\vTheta^{-1} \frac{1}{\nu k}\sum_{i = 1}^k\vDelta_i\right)
     \right).
\end{align*}
The last expression is to be maximized with respect to $\vTheta$ and can be recognized as the MLE problem in a multivariate Gaussian distribution. Hence,
$
  \vTheta = \frac{1}{k \nu} \sum_{i = 1}^k \vDelta_i,
$
is the MLE in this model.

\section{Approximate MLE}
\label{sec:amle}
To find the maximizing parameters we differentiate \eqref{eq:loglik} w.r.t.\ $\vPsi$ and equate to zero while assuming $\nu$ known and constant.
The first order derivative can be seen in equation \eqref{eq:dloglik}.
Equating to zero yields
\begin{align}
  \vec{0}
  &= \frac{k\nu}{2} \vPsi^{-1}
    - \sum_{i=1}^k \frac{\nu + n_i}{2}
      (\vPsi + \vS_i')^{-1}
  \label{eq:firstordderivloglik} \\
  &= \frac{k\nu}{2} \vPsi^{-1}
    - \sum_{i=1}^k \frac{\nu + n_i}{2}
      \left(\vI + \vPsi^{-1}\vS_i\right)^{-1}\vPsi^{-1}.
      \notag
\end{align}
This implies
$ 
  k\nu \vI
    - \sum_{i=1}^k (\nu + n_i)
      \left(\vI - (-\vPsi^{-1}\vS_i)\right)^{-1}
   = \vec{0}
$ 
which can be rewritten as
\begin{align*}
    k\nu \vI
    - \sum_{i=1}^k      (\nu + n_i)
      \sum_{l=0}^\infty \left(-\vPsi^{-1}\vS_i\right)^{l}
   = \vec{0},
\end{align*}
by the Neumann series
$\left((\vI + \vec{A})^{-1} = \sum_{l = 0}^\infty \vec{A}^l\right)$
provided that
$\lim_{l \to \infty} (\vI - \vPsi^{-1}\vS_i)^l = \vec{0}$
for all $i$.
This holds if the eigenvalues of $\vPsi^{-1}\vS_i$ are less than $1$.
We approximate by the first order expansion $(l = 1)$, and
\begin{align*}
  \vec{0}
  = k\nu\vI - \sum_{i=1}^k (\nu + n_i)(\vI - \vPsi^{-1}\vS_i)
  = - n_\bullet\vI
     + \vPsi^{-1}\sum_{i=1}^k (\nu + n_i) \vS_i
\end{align*}
where $n_\bullet = \sum_{i=1}^k n_i$ is the total number of observations.
This implies
\begin{align*}
   \vPsi^{-1}\sum_{i=1}^k (\nu + n_i) \vS_i
    = n_\bullet \vI
\end{align*}
which suggests the estimators
\begin{align}
  \hat{\vPsi}_\text{MLE}
  = \frac{\sum_{i=1}^k (\nu + n_i) \vS_i}{n_\bullet}
  \label{eq:mle}
  \quad \text{ and } \quad
  \hat{\vSigma}_\text{MLE}
  = \frac{\sum_{i=1}^k (\nu + n_i) \vS_i}{(\nu-p-1)n_\bullet}.
\end{align}
These estimates are seen to correspond to a weighted sum of the scatter matrices.

\section{Derivation of ICC}
\label{app:ICC}
Consider observations from \eqref{eq:RCM}.
We temporarily abuse our notation and let
\begin{align*}
  \vSigma \sim \calW_p^{-1}\!( \vPsi, \nu) \quad \text{ and }\quad
  \vS | \vSigma \sim \calW_p(\vSigma, 1),
\end{align*}
and consider only a single observation $(n = 1)$.
Furthermore, let
$\vS = (S_{ij})_{p\times p}$,
$\vSigma = (\Sigma_{ij})_{p \times p}$, and
$\vPsi = (\Psi_{ij})_{p \times p}$.
To compute the ICC, we are thus interested in the ratio of the quantities $\var(\Sigma_{ij})$ and $\var(S_{ij})$ corresponding to the between-study and total variation of the covariance between variables $i$ and $j$, respectively.
That is, the ICC is the proportion of the total variance between studies,
\begin{align}
  \text{ICC}(\nu)
  = \frac{\var(\Sigma_{ij})}
         {\var(S_{ij})}
  = \frac{\var(\Sigma_{ij})}
         {\var(\Sigma_{ij}) + \bbE[ \var(S_{ij}|\vSigma) ]},
  \label{eq:ICC}
\end{align}
where the second equality is obtained by $\bbE[S_{ij}|\vSigma] = \Sigma_{ij}$ and the law of total variation.
This equality agrees with the usual ICC as $\bbE[\var(S_{ij}|\Sigma_{ij})]$ can be interpreted as the (expected) within-study variation.
Using the conditional variance given by $\var(S_{ij}|\vSigma) = \Sigma_{ij}^2 + \Sigma_{ii} \Sigma_{jj}$ the needed quantities can be found.
To compute an expression for \eqref{eq:ICC} we need to consider the fourth-order moments of the observations.
From the model, known results of the inverse Wishart distribution, cf.\ \citep{Cook2011, Rosen1988}, leads to
\begin{align}
  \label{eq:invwishcovar}
  \cov(\Sigma_{ij}, \Sigma_{kl})
  = \frac{2\Psi_{ij}\Psi_{kl}+ (\nu{-}p{-}1)
    \big(\Psi_{ik}\Psi_{jl} + \Psi_{il}\Psi_{kj}\big)}
          {(\nu-p)(\nu-p-1)^2(\nu-p-3)}, \;\nu > p +3,
\end{align}
implying that
\begin{align}
  \var(\Sigma_{ij})
  = \cov(\Sigma_{ij}, \Sigma_{ij})
  \label{eq:invwishvar}
  = \frac{(\nu-p+1)\Psi_{ij}^2 + (\nu-p-1)\Psi_{ii}\Psi_{jj}}
          {(\nu-p)(\nu-p-1)^2(\nu-p-3)}.
\end{align}
Continuing with the expected conditional variance of $S_{ij} | \vSigma$ in the denominator of \eqref{eq:ICC},
\begin{align}
  \bbE\big[\var(S_{ij}|\Sigma_{ij})\big]
  &= \var(\Sigma_{ij}) + \bbE[\Sigma_{ij}]^2
    + \cov(\Sigma_{ii}, \Sigma_{jj})
    + \bbE[\Sigma_{ii}]\bbE[\Sigma_{jj}]  \notag\\
  &= \var(\Sigma_{ij}) + \cov(\Sigma_{ii}, \Sigma_{jj})
    + (\nu-p-1)^{-2}(\Psi_{ij}^2 + \Psi_{ii}\Psi_{jj}).
    \label{eq:varSij}
\end{align}
An expression of $\var(S_{ij})$ in terms of the elements of $\vPsi$ can then found by substituting \eqref{eq:invwishcovar} and \eqref{eq:invwishvar} into \eqref{eq:varSij} and by extension an expression for the ICC \eqref{eq:ICC} can be obtained.
We omit this tedious calculation which can be verified to yield $\text{ICC}(\nu) = 1/(\nu - p)$ as given in \eqref{eq:ICCexprs}

\end{document}

%% file: table1.tex
\begin{table}[!tbp]
{\tiny
\caption{Mean cophenetic correlation and 
            Kullback-Leibler divergence with $95\%$ confidence,
            for estimated  vs true network for 
            different values of $\nu$ and $n_i$ using the EM or Pool method\label{tab:results.clustering}} 
\begin{center}
\begin{tabular}{rrcllcll}
\hline\hline
\multicolumn{2}{c}{\bfseries }&\multicolumn{1}{c}{\bfseries }&\multicolumn{2}{c}{\bfseries Cophenetic Correlation}&\multicolumn{1}{c}{\bfseries }&\multicolumn{2}{c}{\bfseries Kullback-Leibler divergence}\tabularnewline
\cline{1-8}
\multicolumn{1}{c}{$n_i$}&\multicolumn{1}{c}{$\nu$}&\multicolumn{1}{c}{}&\multicolumn{1}{c}{EM}&\multicolumn{1}{c}{Pool}&\multicolumn{1}{c}{}&\multicolumn{1}{c}{EM}&\multicolumn{1}{c}{Pool}\tabularnewline
\hline
$  20$&$   50$&&0.19 (0.17;0.21)&0.2 (0.18;0.22)&&240.37 (232.79;247.94)&227.33 (220.13;234.52)\tabularnewline
$  30$&$   50$&&0.26 (0.23;0.28)&0.25 (0.23;0.28)&&126.61 (123.81;129.41)&121.81 (119.1;124.51)\tabularnewline
$  50$&$   50$&&0.6 (0.56;0.64)&0.43 (0.39;0.46)&&75.62 (73.5;77.74)&73.62 (71.54;75.69)\tabularnewline
$ 100$&$   50$&&0.88 (0.85;0.9)&0.7 (0.67;0.74)&&33.04 (32.56;33.52)&30.9 (30.44;31.36)\tabularnewline
$ 500$&$   50$&&0.99 (0.98;0.99)&0.91 (0.89;0.93)&&23.64 (23.41;23.88)&21.31 (21.1;21.53)\tabularnewline
$1000$&$   50$&&0.99 (0.99;0.99)&0.9 (0.88;0.92)&&22.86 (22.59;23.14)&20.53 (20.28;20.78)\tabularnewline
$  20$&$  100$&&0.35 (0.32;0.38)&0.35 (0.32;0.37)&&76.69 (74.05;79.33)&72.36 (69.85;74.86)\tabularnewline
$  30$&$  100$&&0.4 (0.37;0.42)&0.39 (0.37;0.42)&&34.51 (33.76;35.26)&33.14 (32.42;33.87)\tabularnewline
$  50$&$  100$&&0.72 (0.68;0.75)&0.69 (0.66;0.72)&&27.92 (27.2;28.65)&27.26 (26.55;27.97)\tabularnewline
$ 100$&$  100$&&0.97 (0.96;0.98)&0.96 (0.95;0.97)&&8.02 (7.88;8.16)&7.85 (7.71;7.98)\tabularnewline
$ 500$&$  100$&&1 (0.99;1)&1 (1;1)&&3.34 (3.31;3.38)&3.18 (3.15;3.21)\tabularnewline
$1000$&$  100$&&1 (1;1)&1 (1;1)&&2.95 (2.92;2.98)&2.79 (2.77;2.82)\tabularnewline
$  20$&$ 1000$&&0.51 (0.48;0.54)&0.51 (0.48;0.54)&&52.66 (51.04;54.29)&49.61 (48.07;51.16)\tabularnewline
$  30$&$ 1000$&&0.61 (0.58;0.64)&0.61 (0.58;0.64)&&22.5 (22.05;22.95)&21.59 (21.16;22.02)\tabularnewline
$  50$&$ 1000$&&0.81 (0.78;0.84)&0.81 (0.78;0.84)&&20.49 (19.91;21.08)&20.02 (19.44;20.59)\tabularnewline
$ 100$&$ 1000$&&0.99 (0.98;0.99)&0.99 (0.99;0.99)&&4.47 (4.36;4.58)&4.42 (4.31;4.52)\tabularnewline
$ 500$&$ 1000$&&1 (1;1)&1 (1;1)&&0.71 (0.7;0.72)&0.71 (0.7;0.72)\tabularnewline
$1000$&$ 1000$&&1 (1;1)&1 (1;1)&&0.41 (0.41;0.42)&0.41 (0.4;0.42)\tabularnewline
$  20$&$10000$&&0.53 (0.5;0.55)&0.52 (0.5;0.55)&&53.15 (51.26;55.04)&50.07 (48.28;51.86)\tabularnewline
$  30$&$10000$&&0.65 (0.61;0.68)&0.64 (0.61;0.68)&&21.91 (21.46;22.35)&21.01 (20.59;21.44)\tabularnewline
$  50$&$10000$&&0.83 (0.8;0.85)&0.82 (0.79;0.85)&&19.88 (19.29;20.48)&19.42 (18.84;20.01)\tabularnewline
$ 100$&$10000$&&0.99 (0.99;1)&0.99 (0.99;1)&&4.19 (4.11;4.27)&4.14 (4.06;4.22)\tabularnewline
$ 500$&$10000$&&1 (1;1)&1 (1;1)&&0.59 (0.58;0.6)&0.59 (0.58;0.6)\tabularnewline
$1000$&$10000$&&1 (1;1)&1 (1;1)&&0.28 (0.27;0.28)&0.28 (0.27;0.28)\tabularnewline
\hline
\end{tabular}\end{center}}
\end{table}

%% file: table.top.genes.em.tex
\begin{table}[!tbp]
{\tiny
\begin{center}
\begin{tabular}{lclclclcl}
\hline\hline
\multicolumn{1}{c}{\bfseries Gray}&\multicolumn{1}{c}{\bfseries }&\multicolumn{1}{c}{\bfseries Olivegreen}&\multicolumn{1}{c}{\bfseries }&\multicolumn{1}{c}{\bfseries Orchid}&\multicolumn{1}{c}{\bfseries }&\multicolumn{1}{c}{\bfseries Skyblue}&\multicolumn{1}{c}{\bfseries }&\multicolumn{1}{c}{\bfseries Coral}\tabularnewline
\cline{1-9}
\multicolumn{1}{c}{n = 159}&\multicolumn{1}{c}{}&\multicolumn{1}{c}{n = 50}&\multicolumn{1}{c}{}&\multicolumn{1}{c}{n = 50}&\multicolumn{1}{c}{}&\multicolumn{1}{c}{n = 31}&\multicolumn{1}{c}{}&\multicolumn{1}{c}{n = 10}\tabularnewline
\hline
MYBL1&&FCER1G&&CD2&&COL5A2&&KRT6A\tabularnewline
BATF&&C1QB&&CD3D&&COL1A2&&SPRR1A\tabularnewline
STAP1&&C1QA&&GIMAP4&&COL3A1&&SPRR1B\tabularnewline
CYB5R2&&GBP1&&PTGDS&&THBS2&&KRT13\tabularnewline
TNFRSF13B&&RARRES3&&CCL19&&COL6A3&&SPRR3\tabularnewline
CD44&&IDO1&&CLU&&COL1A1&&S100A2\tabularnewline
MARCKSL1&&CD14&&ADAMDEC1&&COL5A1&&KRT14\tabularnewline
LRMP&&LILRB2&&TRBC2&&VCAN&&DSP\tabularnewline
HCK&&SERPING1&&ITM2A&&FAP&&KRT5\tabularnewline
MME&&PSTPIP2&&LGALS2&&MMP2&&\tabularnewline
LMO2&&GZMA&&ITK&&SULF1&&\tabularnewline
VPREB3&&CCL8&&PLA2G2D&&MXRA5&&\tabularnewline
BCL2A1&&IFNG&&IL7R&&DCN&&\tabularnewline
BLNK&&GBP2&&PLA2G7&&LUM&&\tabularnewline
HLA-DOB&&CXCL10&&ENPP2&&SPARC&&\tabularnewline
RRAS2&&SLAMF7&&IL18&&POSTN&&\tabularnewline
STAG3&&FGL2&&CHI3L1&&COL15A1&&\tabularnewline
BACH2&&CD163&&TFEC&&TMEM45A&&\tabularnewline
CCND2&&CXCL11&&CXCL13&&COL11A1&&\tabularnewline
PDGFD&&GZMH&&CCL21&&CTSK&&\tabularnewline
NCF2&&ALDH1A1&&CSTA&&EMP1&&\tabularnewline
SPINK2&&CXCL9&&MMP9&&AEBP1&&\tabularnewline
MNDA&&GZMK&&LYZ&&TGFBI&&\tabularnewline
MS4A1&&GZMB&&HSD11B1&&GJA1&&\tabularnewline
CD22&&KCNJ2&&APOC1&&PLS3&&\tabularnewline
OSBPL10&&CPVL&&CXCL14&&TIMP1&&\tabularnewline
GPR137B&&IGSF6&&C3&&ANXA1&&\tabularnewline
GRHPR&&LGMN&&MAL&&TNFAIP6&&\tabularnewline
SORL1&&MT2A&&CYP27B1&&SPP1&&\tabularnewline
IGF2BP3&&MT1G&&LAMP3&&&&\tabularnewline
SYBU&&CD8A&&CHIT1&&&&\tabularnewline
TCL1A&&MS4A4A&&PLAC8&&&&\tabularnewline
ZNF804A&&CRTAM&&SELL&&&&\tabularnewline
SLC12A8&&S100A9&&KLRB1&&&&\tabularnewline
CTGF&&MARCO&&CD69&&&&\tabularnewline
FCRL2&&S100A8&&ROBO1&&&&\tabularnewline
DUSP5&&MT1M&&ORM1&&&&\tabularnewline
CCR10&&GPX3&&S1PR1&&&&\tabularnewline
ALOX5AP&&GNLY&&CCR7&&&&\tabularnewline
RGCC&&MT1E&&GPR183&&&&\tabularnewline
\hline
\end{tabular}
\caption{The identified modules, their sizes, and member genes. The genes are sorted decreasingly by their intra-module connectivity (sum of the incident edge weights). Only the top 40 genes are shown.\label{tab:top.genes.em}}\end{center}}
\end{table}

%% file: tableS2.tex
\begin{landscape}\begin{table}[!tbp]
{\tiny
\caption{Mean cophenetic correlation and 
Kullback-Leibler divergence with $95\%$ confidence,
for estimated  vs true network for 
different values of $\nu$ and $n_i$ using the EM, MLE or Pool method\label{tab:results.clustering.full}} 
\begin{center}
\begin{tabular}{rrclllclll}
\hline\hline
\multicolumn{2}{c}{\bfseries }&\multicolumn{1}{c}{\bfseries }&\multicolumn{3}{c}{\bfseries Cophenetic Correlation}&\multicolumn{1}{c}{\bfseries }&\multicolumn{3}{c}{\bfseries Kullback-Leibler divergence}\tabularnewline
\cline{1-10}
\multicolumn{1}{c}{$n_i$}&\multicolumn{1}{c}{$\nu$}&\multicolumn{1}{c}{}&\multicolumn{1}{c}{EM}&\multicolumn{1}{c}{MLE}&\multicolumn{1}{c}{Pool}&\multicolumn{1}{c}{}&\multicolumn{1}{c}{EM}&\multicolumn{1}{c}{MLE}&\multicolumn{1}{c}{Pool}\tabularnewline
\hline
$  20$&$   50$&&0.19 (0.17;0.21)&0.2 (0.18;0.22)&0.2 (0.18;0.22)&&236.64 (229.16;244.11)&240.37 (232.79;247.94)&227.33 (220.13;234.52)\tabularnewline
$  30$&$   50$&&0.26 (0.23;0.28)&0.25 (0.23;0.28)&0.25 (0.23;0.28)&&123.99 (121.24;126.74)&126.61 (123.81;129.41)&121.81 (119.1;124.51)\tabularnewline
$  50$&$   50$&&0.6 (0.56;0.64)&0.43 (0.39;0.46)&0.43 (0.39;0.46)&&44.95 (43.34;46.56)&75.62 (73.5;77.74)&73.62 (71.54;75.69)\tabularnewline
$ 100$&$   50$&&0.88 (0.85;0.9)&0.7 (0.67;0.74)&0.7 (0.67;0.74)&&13.15 (12.81;13.49)&33.04 (32.56;33.52)&30.9 (30.44;31.36)\tabularnewline
$ 500$&$   50$&&0.99 (0.98;0.99)&0.91 (0.89;0.93)&0.91 (0.89;0.93)&&7.36 (7.18;7.55)&23.64 (23.41;23.88)&21.31 (21.1;21.53)\tabularnewline
$1000$&$   50$&&0.99 (0.99;0.99)&0.9 (0.88;0.92)&0.9 (0.88;0.92)&&6.87 (6.69;7.06)&22.86 (22.59;23.14)&20.53 (20.28;20.78)\tabularnewline
$  20$&$  100$&&0.35 (0.32;0.38)&0.35 (0.32;0.37)&0.35 (0.32;0.37)&&75.67 (73.06;78.28)&76.69 (74.05;79.33)&72.36 (69.85;74.86)\tabularnewline
$  30$&$  100$&&0.4 (0.37;0.42)&0.39 (0.37;0.42)&0.39 (0.37;0.42)&&34.01 (33.27;34.75)&34.51 (33.76;35.26)&33.14 (32.42;33.87)\tabularnewline
$  50$&$  100$&&0.72 (0.68;0.75)&0.69 (0.66;0.72)&0.69 (0.66;0.72)&&27.47 (26.75;28.18)&27.92 (27.2;28.65)&27.26 (26.55;27.97)\tabularnewline
$ 100$&$  100$&&0.97 (0.96;0.98)&0.96 (0.95;0.97)&0.96 (0.95;0.97)&&6.75 (6.62;6.88)&8.02 (7.88;8.16)&7.85 (7.71;7.98)\tabularnewline
$ 500$&$  100$&&1 (0.99;1)&1 (1;1)&1 (1;1)&&2.37 (2.35;2.4)&3.34 (3.31;3.38)&3.18 (3.15;3.21)\tabularnewline
$1000$&$  100$&&1 (1;1)&1 (1;1)&1 (1;1)&&2.02 (2;2.04)&2.95 (2.92;2.98)&2.79 (2.77;2.82)\tabularnewline
$  20$&$ 1000$&&0.51 (0.48;0.54)&0.51 (0.48;0.54)&0.51 (0.48;0.54)&&52.09 (50.49;53.7)&52.66 (51.04;54.29)&49.61 (48.07;51.16)\tabularnewline
$  30$&$ 1000$&&0.61 (0.58;0.64)&0.61 (0.58;0.64)&0.61 (0.58;0.64)&&22.29 (21.85;22.74)&22.5 (22.05;22.95)&21.59 (21.16;22.02)\tabularnewline
$  50$&$ 1000$&&0.81 (0.78;0.84)&0.81 (0.78;0.84)&0.81 (0.78;0.84)&&20.34 (19.76;20.92)&20.49 (19.91;21.08)&20.02 (19.44;20.59)\tabularnewline
$ 100$&$ 1000$&&0.99 (0.98;0.99)&0.99 (0.99;0.99)&0.99 (0.99;0.99)&&4.41 (4.3;4.51)&4.47 (4.36;4.58)&4.42 (4.31;4.52)\tabularnewline
$ 500$&$ 1000$&&1 (1;1)&1 (1;1)&1 (1;1)&&0.69 (0.68;0.7)&0.71 (0.7;0.72)&0.71 (0.7;0.72)\tabularnewline
$1000$&$ 1000$&&1 (1;1)&1 (1;1)&1 (1;1)&&0.4 (0.4;0.41)&0.41 (0.41;0.42)&0.41 (0.4;0.42)\tabularnewline
$  20$&$10000$&&0.53 (0.5;0.55)&0.52 (0.5;0.55)&0.52 (0.5;0.55)&&52.58 (50.71;54.46)&53.15 (51.26;55.04)&50.07 (48.28;51.86)\tabularnewline
$  30$&$10000$&&0.65 (0.61;0.68)&0.64 (0.61;0.68)&0.64 (0.61;0.68)&&21.71 (21.27;22.15)&21.91 (21.46;22.35)&21.01 (20.59;21.44)\tabularnewline
$  50$&$10000$&&0.83 (0.8;0.85)&0.82 (0.79;0.85)&0.82 (0.79;0.85)&&19.75 (19.16;20.34)&19.88 (19.29;20.48)&19.42 (18.84;20.01)\tabularnewline
$ 100$&$10000$&&0.99 (0.99;1)&0.99 (0.99;1)&0.99 (0.99;1)&&4.13 (4.05;4.21)&4.19 (4.11;4.27)&4.14 (4.06;4.22)\tabularnewline
$ 500$&$10000$&&1 (1;1)&1 (1;1)&1 (1;1)&&0.58 (0.57;0.59)&0.59 (0.58;0.6)&0.59 (0.58;0.6)\tabularnewline
$1000$&$10000$&&1 (1;1)&1 (1;1)&1 (1;1)&&0.27 (0.27;0.28)&0.28 (0.27;0.28)&0.28 (0.27;0.28)\tabularnewline
\hline
\end{tabular}\end{center}}
\end{table}\end{landscape}

%% file: tableS1.tex
\begin{landscape}\begin{table}[!tbp]
{\tiny
\caption{Simulation results based on IDRC data. Mean cophenetic correlation and 
            Kullback-Leibler divergence  with $95\%$ confidence,
            for estimated vs true network for 
            different values of $\nu$ and $n_i$ using the EM, MLE or Pool method.\label{tab:results.clustering.idrc}} 
\begin{center}
\begin{tabular}{rrclllclll}
\hline\hline
\multicolumn{2}{c}{\bfseries }&\multicolumn{1}{c}{\bfseries }&\multicolumn{3}{c}{\bfseries Cophenetic Correlation}&\multicolumn{1}{c}{\bfseries }&\multicolumn{3}{c}{\bfseries Kullback-Leibler divergence}\tabularnewline
\cline{1-10}
\multicolumn{1}{c}{$n_i$}&\multicolumn{1}{c}{$\nu$}&\multicolumn{1}{c}{}&\multicolumn{1}{c}{EM}&\multicolumn{1}{c}{MLE}&\multicolumn{1}{c}{Pool}&\multicolumn{1}{c}{}&\multicolumn{1}{c}{EM}&\multicolumn{1}{c}{MLE}&\multicolumn{1}{c}{Pool}\tabularnewline
\hline
$  40$&$  150$&&0.22 (0.2;0.25)&0.22 (0.2;0.25)&0.22 (0.2;0.25)&&1420.79 (1384.39;1457.2)&1439.26 (1402.35;1476.17)&1402.29 (1366.3;1438.28)\tabularnewline
$  50$&$  150$&&0.25 (0.23;0.27)&0.25 (0.23;0.28)&0.25 (0.23;0.28)&&647.2 (637.66;656.73)&656.21 (646.56;665.85)&642.39 (632.93;651.84)\tabularnewline
$  70$&$  150$&&0.3 (0.27;0.32)&0.29 (0.27;0.31)&0.29 (0.27;0.31)&&383.7 (379.35;388.04)&389.72 (385.3;394.14)&383.75 (379.4;388.1)\tabularnewline
$  90$&$  150$&&0.32 (0.3;0.34)&0.32 (0.3;0.34)&0.32 (0.3;0.34)&&306.15 (303.31;309)&311.8 (308.89;314.71)&308.05 (305.18;310.93)\tabularnewline
$ 150$&$  150$&&0.73 (0.72;0.74)&0.65 (0.62;0.67)&0.65 (0.62;0.67)&&55.57 (55.07;56.08)&82.27 (81.62;82.93)&80.52 (79.88;81.17)\tabularnewline
$ 500$&$  150$&&0.79 (0.78;0.8)&0.76 (0.74;0.77)&0.76 (0.74;0.77)&&16.82 (16.68;16.95)&31.61 (31.44;31.78)&29.51 (29.35;29.67)\tabularnewline
$1000$&$  150$&&0.82 (0.81;0.84)&0.79 (0.77;0.8)&0.79 (0.77;0.8)&&12.6 (12.51;12.68)&26.32 (26.19;26.45)&24.23 (24.11;24.35)\tabularnewline
$  40$&$  200$&&0.29 (0.26;0.32)&0.29 (0.27;0.32)&0.29 (0.27;0.32)&&1025.53 (997.03;1054.02)&1038.25 (1009.41;1067.09)&1011.46 (983.34;1039.58)\tabularnewline
$  50$&$  200$&&0.3 (0.28;0.32)&0.29 (0.27;0.32)&0.29 (0.27;0.32)&&446.39 (439.86;452.93)&452.07 (445.46;458.68)&442.48 (436;448.96)\tabularnewline
$  70$&$  200$&&0.36 (0.33;0.38)&0.35 (0.33;0.38)&0.35 (0.33;0.38)&&258.25 (255.7;260.79)&261.7 (259.11;264.28)&257.67 (255.13;260.21)\tabularnewline
$  90$&$  200$&&0.38 (0.36;0.41)&0.38 (0.36;0.4)&0.38 (0.36;0.4)&&203 (200.99;205.01)&205.82 (203.78;207.86)&203.34 (201.32;205.36)\tabularnewline
$ 150$&$  200$&&0.76 (0.75;0.77)&0.72 (0.7;0.73)&0.72 (0.7;0.73)&&48.6 (48.21;48.99)&58.25 (57.79;58.71)&57.71 (57.26;58.17)\tabularnewline
$ 500$&$  200$&&0.8 (0.78;0.81)&0.8 (0.79;0.81)&0.8 (0.79;0.81)&&13.45 (13.35;13.55)&18.23 (18.13;18.32)&17.41 (17.32;17.5)\tabularnewline
$1000$&$  200$&&0.83 (0.82;0.84)&0.81 (0.8;0.83)&0.81 (0.8;0.83)&&9.3 (9.24;9.36)&13.84 (13.78;13.91)&13.06 (13;13.12)\tabularnewline
$  40$&$ 1000$&&0.38 (0.36;0.41)&0.38 (0.36;0.4)&0.38 (0.36;0.4)&&651.25 (633.44;669.06)&658.49 (640.51;676.47)&641.33 (623.8;658.86)\tabularnewline
$  50$&$ 1000$&&0.41 (0.39;0.44)&0.41 (0.38;0.43)&0.41 (0.38;0.43)&&285.71 (281.84;289.58)&288.67 (284.77;292.57)&282.48 (278.66;286.3)\tabularnewline
$  70$&$ 1000$&&0.46 (0.44;0.47)&0.45 (0.43;0.47)&0.45 (0.43;0.47)&&161 (159.25;162.74)&162.46 (160.7;164.23)&159.95 (158.21;161.69)\tabularnewline
$  90$&$ 1000$&&0.47 (0.45;0.49)&0.47 (0.45;0.48)&0.47 (0.45;0.48)&&126.3 (125.2;127.39)&127.33 (126.22;128.43)&125.8 (124.71;126.9)\tabularnewline
$ 150$&$ 1000$&&0.74 (0.73;0.76)&0.72 (0.71;0.74)&0.72 (0.71;0.74)&&38.4 (38.11;38.7)&38.96 (38.66;39.26)&38.71 (38.41;39.01)\tabularnewline
$ 500$&$ 1000$&&0.84 (0.83;0.85)&0.84 (0.83;0.85)&0.84 (0.83;0.85)&&7.92 (7.87;7.98)&8.11 (8.06;8.16)&8.05 (8;8.1)\tabularnewline
$1000$&$ 1000$&&0.88 (0.87;0.89)&0.86 (0.85;0.87)&0.86 (0.85;0.87)&&4.37 (4.34;4.4)&4.45 (4.42;4.48)&4.41 (4.38;4.44)\tabularnewline
$  40$&$10000$&&0.4 (0.38;0.42)&0.4 (0.38;0.41)&0.4 (0.38;0.41)&&640.15 (623.66;656.64)&647.08 (630.41;663.76)&630.21 (613.95;646.47)\tabularnewline
$  50$&$10000$&&0.42 (0.4;0.44)&0.42 (0.4;0.44)&0.42 (0.4;0.44)&&272.28 (268.89;275.67)&274.94 (271.52;278.37)&269.04 (265.68;272.39)\tabularnewline
$  70$&$10000$&&0.46 (0.44;0.48)&0.46 (0.44;0.48)&0.46 (0.44;0.48)&&153.95 (152.4;155.5)&155.23 (153.66;156.79)&152.82 (151.28;154.36)\tabularnewline
$  90$&$10000$&&0.47 (0.45;0.49)&0.46 (0.45;0.48)&0.46 (0.45;0.48)&&121.16 (120.12;122.2)&122.04 (121;123.09)&120.58 (119.55;121.62)\tabularnewline
$ 150$&$10000$&&0.75 (0.73;0.76)&0.74 (0.73;0.76)&0.74 (0.73;0.76)&&36.83 (36.5;37.16)&37.21 (36.88;37.55)&36.98 (36.65;37.31)\tabularnewline
$ 500$&$10000$&&0.85 (0.84;0.86)&0.84 (0.83;0.86)&0.84 (0.83;0.86)&&7.07 (7.03;7.12)&7.19 (7.14;7.23)&7.16 (7.12;7.21)\tabularnewline
$1000$&$10000$&&0.87 (0.86;0.88)&0.85 (0.84;0.86)&0.85 (0.84;0.86)&&3.49 (3.46;3.51)&3.52 (3.49;3.54)&3.51 (3.49;3.53)\tabularnewline
\hline
\end{tabular}\end{center}}
\end{table}\end{landscape}

%% file: table.enrichment.em.tex
\setlongtables{\tiny
\begin{longtable}{lllrrr}\caption{The significant terms for the gene enrichment ,
                                          analysis of the DLBCL EM method modules. Number of genes ,
                                          in each term (N), and the overlap to module (O).} \tabularnewline
\hline\hline
\multicolumn{1}{l}{Term ID}&\multicolumn{1}{c}{Domain}&\multicolumn{1}{c}{Term}&\multicolumn{1}{c}{P}&\multicolumn{1}{c}{N}&\multicolumn{1}{c}{O}\tabularnewline
\hline
\endfirsthead\caption[]{\em (continued)} \tabularnewline*
\hline
\multicolumn{1}{l}{Term ID}&\multicolumn{1}{c}{Domain}&\multicolumn{1}{c}{Term}&\multicolumn{1}{c}{P}&\multicolumn{1}{c}{N}&\multicolumn{1}{c}{O}\tabularnewline
\hline
\endhead
\hline
\endfoot
\label{tab:enrichment.em}
{\bfseries Gray}&&&&&\tabularnewline*
~~$GO:0097159$&MF&organic cyclic compound bindin&$5.55e-03$&$157$&$ 50$\tabularnewline*
~~$GO:1901363$&MF&heterocyclic compound binding&$6.64e-03$&$157$&$ 48$\tabularnewline*
~~$TF:M00940_1$&tf&Factor: E2F-1; motif: NTTTCGCG&$1.50e-02$&$157$&$ 34$\tabularnewline*
~~$TF:M00695_0$&tf&Factor: ETF; motif: GVGGMGG; m&$2.38e-03$&$157$&$ 78$\tabularnewline*
~~$TF:M00428_1$&tf&Factor: E2F-1; motif: NKTSSCGC&$1.23e-03$&$157$&$ 56$\tabularnewline*
~~$TF:M03807_1$&tf&Factor: SP2; motif: GNNGGGGGCG&$4.27e-02$&$157$&$ 40$\tabularnewline*
~~$TF:M01199_0$&tf&Factor: RNF96; motif: BCCCGCRG&$3.63e-02$&$157$&$ 65$\tabularnewline*
~~$TF:M01199_1$&tf&Factor: RNF96; motif: BCCCGCRG&$4.27e-02$&$157$&$ 40$\tabularnewline*
~~$TF:M04869_0$&tf&Factor: Egr-1; motif: GCGCATGC&$2.42e-02$&$157$&$ 99$\tabularnewline*
~~$TF:M04869_1$&tf&Factor: Egr-1; motif: GCGCATGC&$1.87e-03$&$157$&$ 88$\tabularnewline*
~~$TF:M07250_1$&tf&Factor: E2F-1; motif: NNNSSCGC&$3.95e-02$&$157$&$ 54$\tabularnewline*
~~$TF:M08874_1$&tf&Factor: E2F1; motif: NNNNNGCGS&$1.61e-02$&$157$&$ 36$\tabularnewline*
~~$TF:M00803_1$&tf&Factor: E2F; motif: GGCGSG; ma&$3.88e-04$&$157$&$ 87$\tabularnewline*
~~$TF:M00938_1$&tf&Factor: E2F-1; motif: TTGGCGCG&$5.69e-03$&$157$&$ 40$\tabularnewline*
~~$TF:M02090_0$&tf&Factor: E2F-4; motif: GCGGGAAA&$1.38e-02$&$157$&$117$\tabularnewline
\hline
{\bfseries Olivegreen}&&&&&\tabularnewline*
~~$GO:0006952$&BP&defense response&$7.88e-04$&$ 48$&$ 29$\tabularnewline*
~~$GO:0045087$&BP&innate immune response&$3.99e-04$&$ 48$&$ 20$\tabularnewline*
~~$GO:0045088$&BP&regulation of innate immune re&$2.33e-02$&$ 48$&$ 10$\tabularnewline*
~~$GO:0009615$&BP&response to virus&$4.94e-02$&$ 48$&$  9$\tabularnewline*
~~$GO:0009607$&BP&response to biotic stimulus&$2.64e-02$&$ 48$&$ 18$\tabularnewline*
~~$REAC:5660526$&rea&Response to metal ions&$5.00e-02$&$ 48$&$  4$\tabularnewline*
~~$REAC:5661231$&rea&Metallothioneins bind metals&$5.00e-02$&$ 48$&$  4$\tabularnewline
\hline
{\bfseries Orchid}&&&&&\tabularnewline*
~~$GO:0002376$&BP&immune system process&$2.90e-02$&$ 49$&$ 36$\tabularnewline*
~~$GO:0006644$&BP&phospholipid metabolic process&$1.83e-02$&$ 49$&$  8$\tabularnewline*
~~$REAC:2730905$&rea&Role of LAT2/NTAL/LAB on calci&$1.08e-02$&$ 49$&$  5$\tabularnewline*
~~$REAC:2029482$&rea&Regulation of actin dynamics f&$4.98e-02$&$ 49$&$  4$\tabularnewline*
~~$REAC:2029485$&rea&Role of phospholipids in phago&$4.98e-02$&$ 49$&$  4$\tabularnewline*
~~$REAC:2871809$&rea&FCERI mediated Ca+2 mobilizati&$1.08e-02$&$ 49$&$  5$\tabularnewline*
~~$REAC:2871837$&rea&FCERI mediated NF-kB activatio&$4.98e-02$&$ 49$&$  4$\tabularnewline
\hline
\newpage
{\bfseries Skyblue}&&&&&\tabularnewline*
~~$GO:0009719$&BP&response to endogenous stimulu&$8.88e-03$&$ 31$&$ 12$\tabularnewline*
~~$GO:0009611$&BP&response to wounding&$4.05e-03$&$ 31$&$ 12$\tabularnewline*
~~$GO:0042060$&BP&wound healing&$9.50e-03$&$ 31$&$ 11$\tabularnewline*
~~$GO:0071230$&BP&cellular response to amino aci&$2.58e-02$&$ 31$&$  5$\tabularnewline*
~~$GO:0032502$&BP&developmental process&$1.34e-02$&$ 31$&$ 26$\tabularnewline*
~~$GO:0048856$&BP&anatomical structure developme&$2.78e-03$&$ 31$&$ 26$\tabularnewline*
~~$GO:0009888$&BP&tissue development&$1.16e-02$&$ 31$&$ 16$\tabularnewline*
~~$GO:0009653$&BP&anatomical structure morphogen&$2.64e-03$&$ 31$&$ 18$\tabularnewline*
~~$GO:0051093$&BP&negative regulation of develop&$1.07e-03$&$ 31$&$ 12$\tabularnewline*
~~$GO:0048646$&BP&anatomical structure formation&$3.52e-02$&$ 31$&$ 12$\tabularnewline*
~~$GO:0044767$&BP&single-organism developmental &$9.58e-03$&$ 31$&$ 26$\tabularnewline*
~~$GO:0007275$&BP&multicellular organism develop&$4.02e-04$&$ 31$&$ 26$\tabularnewline*
~~$GO:0048731$&BP&system development&$2.28e-03$&$ 31$&$ 24$\tabularnewline*
~~$GO:0072359$&BP&circulatory system development&$1.83e-04$&$ 31$&$ 14$\tabularnewline*
~~$GO:0001501$&BP&skeletal system development&$8.37e-09$&$ 31$&$ 15$\tabularnewline*
~~$GO:0072358$&BP&cardiovascular system developm&$2.66e-03$&$ 31$&$ 12$\tabularnewline*
~~$GO:0001944$&BP&vasculature development&$2.66e-03$&$ 31$&$ 12$\tabularnewline*
~~$GO:0001568$&BP&blood vessel development&$2.66e-03$&$ 31$&$ 12$\tabularnewline*
~~$GO:0010243$&BP&response to organonitrogen com&$3.02e-02$&$ 31$&$  9$\tabularnewline*
~~$GO:0032501$&BP&multicellular organismal proce&$2.06e-02$&$ 31$&$ 28$\tabularnewline*
~~$GO:0071840$&BP&cellular component organizatio&$1.36e-03$&$ 31$&$ 25$\tabularnewline*
~~$GO:0016043$&BP&cellular component organizatio&$1.11e-03$&$ 31$&$ 25$\tabularnewline*
~~$GO:0043062$&BP&extracellular structure organi&$1.77e-13$&$ 31$&$ 20$\tabularnewline*
~~$GO:0030198$&BP&extracellular matrix organizat&$1.77e-13$&$ 31$&$ 20$\tabularnewline*
~~$GO:0030199$&BP&collagen fibril organization&$1.13e-03$&$ 31$&$  7$\tabularnewline*
~~$GO:0050953$&BP&sensory perception of light st&$4.98e-02$&$ 31$&$  4$\tabularnewline*
~~$GO:0007601$&BP&visual perception&$4.98e-02$&$ 31$&$  4$\tabularnewline*
~~$GO:0009056$&BP&catabolic process&$6.84e-03$&$ 31$&$ 16$\tabularnewline*
~~$GO:0044712$&BP&single-organism catabolic proc&$1.14e-04$&$ 31$&$ 14$\tabularnewline*
~~$GO:0044236$&BP&multicellular organism metabol&$2.82e-05$&$ 31$&$ 11$\tabularnewline*
~~$GO:0044243$&BP&multicellular organismal catab&$1.48e-06$&$ 31$&$ 11$\tabularnewline*
~~$GO:0044259$&BP&multicellular organismal macro&$2.82e-05$&$ 31$&$ 11$\tabularnewline*
~~$GO:0032963$&BP&collagen metabolic process&$2.82e-05$&$ 31$&$ 11$\tabularnewline*
~~$GO:0030574$&BP&collagen catabolic process&$1.48e-06$&$ 31$&$ 11$\tabularnewline*
~~$GO:0012505$&CC&endomembrane system&$1.91e-03$&$ 31$&$ 23$\tabularnewline*
~~$GO:0005576$&CC&extracellular region&$2.40e-03$&$ 31$&$ 28$\tabularnewline*
~~$GO:0099080$&CC&supramolecular complex&$2.43e-03$&$ 31$&$ 11$\tabularnewline*
~~$GO:0099081$&CC&supramolecular polymer&$2.43e-03$&$ 31$&$ 11$\tabularnewline*
~~$GO:0099512$&CC&supramolecular fiber&$2.43e-03$&$ 31$&$ 11$\tabularnewline*
~~$GO:0044421$&CC&extracellular region part&$1.18e-04$&$ 31$&$ 28$\tabularnewline*
~~$GO:0005615$&CC&extracellular space&$9.33e-06$&$ 31$&$ 25$\tabularnewline*
~~$GO:0031012$&CC&extracellular matrix&$4.04e-10$&$ 31$&$ 19$\tabularnewline*
~~$GO:0044420$&CC&extracellular matrix component&$6.83e-12$&$ 31$&$ 13$\tabularnewline*
~~$GO:0005578$&CC&proteinaceous extracellular ma&$3.49e-11$&$ 31$&$ 18$\tabularnewline*
~~$GO:0005604$&CC&basement membrane&$3.66e-05$&$ 31$&$  7$\tabularnewline*
~~$GO:0043234$&CC&protein complex&$2.49e-03$&$ 31$&$ 13$\tabularnewline*
~~$GO:0005581$&CC&collagen trimer&$4.24e-07$&$ 31$&$ 11$\tabularnewline*
~~$GO:0098644$&CC&complex of collagen trimers&$3.66e-05$&$ 31$&$  7$\tabularnewline*
~~$GO:0098643$&CC&banded collagen fibril&$3.66e-05$&$ 31$&$  7$\tabularnewline*
~~$GO:0005583$&CC&fibrillar collagen trimer&$3.66e-05$&$ 31$&$  7$\tabularnewline*
~~$GO:0044432$&CC&endoplasmic reticulum part&$8.73e-03$&$ 31$&$ 10$\tabularnewline*
~~$GO:0005788$&CC&endoplasmic reticulum lumen&$2.57e-05$&$ 31$&$  8$\tabularnewline*
~~$GO:0050840$&MF&extracellular matrix binding&$2.58e-02$&$ 31$&$  5$\tabularnewline*
~~$GO:0048407$&MF&platelet-derived growth factor&$4.98e-02$&$ 31$&$  4$\tabularnewline*
~~$GO:0043169$&MF&cation binding&$1.18e-03$&$ 31$&$ 18$\tabularnewline*
~~$GO:0046872$&MF&metal ion binding&$4.93e-04$&$ 31$&$ 18$\tabularnewline*
~~$GO:0005201$&MF&extracellular matrix structura&$1.18e-05$&$ 31$&$  9$\tabularnewline*
~~$GO:0044877$&MF&macromolecular complex binding&$3.26e-02$&$ 31$&$ 10$\tabularnewline*
~~$GO:0032403$&MF&protein complex binding&$2.16e-02$&$ 31$&$ 10$\tabularnewline*
~~$HP:0000002$&hp&Abnormality of body height&$4.37e-02$&$ 31$&$ 12$\tabularnewline*
~~$KEGG:04510$&keg&Focal adhesion&$4.99e-02$&$ 31$&$  5$\tabularnewline*
~~$KEGG:04974$&keg&Protein digestion and absorpti&$4.45e-04$&$ 31$&$  8$\tabularnewline*
~~$KEGG:04512$&keg&ECM-receptor interaction&$1.65e-02$&$ 31$&$  5$\tabularnewline*
~~$REAC:3781865$&rea&Diseases of glycosylation&$7.07e-03$&$ 31$&$  4$\tabularnewline*
~~$REAC:1474244$&rea&Extracellular matrix organizat&$6.26e-11$&$ 31$&$ 17$\tabularnewline*
~~$REAC:3000178$&rea&ECM proteoglycans&$7.07e-03$&$ 31$&$  4$\tabularnewline*
~~$REAC:1474228$&rea&Degradation of the extracellul&$3.07e-02$&$ 31$&$  6$\tabularnewline*
~~$REAC:1474290$&rea&Collagen formation&$2.98e-05$&$ 31$&$  9$\tabularnewline*
~~$REAC:2022090$&rea&Assembly of collagen fibrils a&$1.97e-04$&$ 31$&$  8$\tabularnewline*
~~$REAC:1650814$&rea&Collagen biosynthesis and modi&$1.85e-07$&$ 31$&$  9$\tabularnewline
\hline
\newpage
{\bfseries Coral}&&&&&\tabularnewline*
~~$GO:0009888$&BP&tissue development&$1.02e-03$&$ 10$&$  9$\tabularnewline*
~~$GO:0060429$&BP&epithelium development&$7.35e-07$&$ 10$&$  9$\tabularnewline*
~~$GO:0030855$&BP&epithelial cell differentiatio&$1.01e-08$&$ 10$&$  9$\tabularnewline*
~~$GO:0008544$&BP&epidermis development&$2.23e-10$&$ 10$&$  9$\tabularnewline*
~~$GO:0009913$&BP&epidermal cell differentiation&$2.47e-11$&$ 10$&$  9$\tabularnewline*
~~$GO:0008219$&BP&cell death&$3.95e-03$&$ 10$&$  9$\tabularnewline*
~~$GO:0012501$&BP&programmed cell death&$2.22e-03$&$ 10$&$  9$\tabularnewline*
~~$GO:0043588$&BP&skin development&$5.35e-09$&$ 10$&$  9$\tabularnewline*
~~$GO:0030216$&BP&keratinocyte differentiation&$2.47e-11$&$ 10$&$  9$\tabularnewline*
~~$GO:0031424$&BP&keratinization&$1.13e-12$&$ 10$&$  9$\tabularnewline*
~~$GO:0070268$&BP&cornification&$1.44e-10$&$ 10$&$  8$\tabularnewline*
~~$GO:0018149$&BP&peptide cross-linking&$3.55e-02$&$ 10$&$  4$\tabularnewline*
~~$GO:0099513$&CC&polymeric cytoskeletal fiber&$4.54e-03$&$ 10$&$  5$\tabularnewline*
~~$GO:0045111$&CC&intermediate filament cytoskel&$3.76e-05$&$ 10$&$  5$\tabularnewline*
~~$GO:0005882$&CC&intermediate filament&$6.36e-06$&$ 10$&$  5$\tabularnewline*
~~$GO:0045095$&CC&keratin filament&$3.06e-04$&$ 10$&$  4$\tabularnewline*
~~$GO:0001533$&CC&cornified envelope&$4.44e-03$&$ 10$&$  4$\tabularnewline*
~~$GO:0005198$&MF&structural molecule activity&$4.41e-05$&$ 10$&$  8$\tabularnewline*
~~$GO:0005200$&MF&structural constituent of cyto&$2.00e-02$&$ 10$&$  4$\tabularnewline*
~~$HPA:053030_01$&hpa&tonsil; squamous epithelial ce&$9.03e-03$&$ 10$&$  9$\tabularnewline*
~~$HPA:053030_02$&hpa&tonsil; squamous epithelial ce&$1.63e-04$&$ 10$&$  9$\tabularnewline*
~~$HPA:053030_11$&hpa&tonsil; squamous epithelial ce&$7.87e-05$&$ 10$&$  9$\tabularnewline*
~~$HPA:053030_12$&hpa&tonsil; squamous epithelial ce&$2.70e-06$&$ 10$&$  9$\tabularnewline*
~~$HPA:029010_01$&hpa&oral mucosa; squamous epitheli&$2.59e-03$&$ 10$&$  9$\tabularnewline*
~~$HPA:029010_02$&hpa&oral mucosa; squamous epitheli&$8.70e-04$&$ 10$&$  8$\tabularnewline*
~~$HPA:029010_03$&hpa&oral mucosa; squamous epitheli&$9.38e-03$&$ 10$&$  5$\tabularnewline*
~~$HPA:029010_11$&hpa&oral mucosa; squamous epitheli&$3.54e-05$&$ 10$&$  9$\tabularnewline*
~~$HPA:029010_12$&hpa&oral mucosa; squamous epitheli&$8.81e-05$&$ 10$&$  8$\tabularnewline*
~~$HPA:029010_13$&hpa&oral mucosa; squamous epitheli&$9.38e-03$&$ 10$&$  5$\tabularnewline*
~~$HPA:015010_01$&hpa&esophagus; squamous epithelial&$3.06e-03$&$ 10$&$  9$\tabularnewline*
~~$HPA:015010_02$&hpa&esophagus; squamous epithelial&$1.29e-04$&$ 10$&$  9$\tabularnewline*
~~$HPA:015010_03$&hpa&esophagus; squamous epithelial&$9.55e-05$&$ 10$&$  7$\tabularnewline*
~~$HPA:015010_11$&hpa&esophagus; squamous epithelial&$4.66e-05$&$ 10$&$  9$\tabularnewline*
~~$HPA:015010_12$&hpa&esophagus; squamous epithelial&$1.84e-06$&$ 10$&$  9$\tabularnewline*
~~$HPA:015010_13$&hpa&esophagus; squamous epithelial&$3.52e-05$&$ 10$&$  7$\tabularnewline*
~~$HPA:009020_01$&hpa&cervix, uterine; squamous epit&$1.46e-02$&$ 10$&$  8$\tabularnewline*
~~$HPA:009020_02$&hpa&cervix, uterine; squamous epit&$3.99e-04$&$ 10$&$  8$\tabularnewline*
~~$HPA:009020_03$&hpa&cervix, uterine; squamous epit&$1.30e-03$&$ 10$&$  6$\tabularnewline*
~~$HPA:009020_11$&hpa&cervix, uterine; squamous epit&$3.02e-04$&$ 10$&$  8$\tabularnewline*
~~$HPA:009020_12$&hpa&cervix, uterine; squamous epit&$1.33e-05$&$ 10$&$  8$\tabularnewline*
~~$HPA:009020_13$&hpa&cervix, uterine; squamous epit&$8.27e-04$&$ 10$&$  6$\tabularnewline*
~~$HPA:042030_11$&hpa&skin 1; keratinocytes[Supporte&$1.62e-02$&$ 10$&$  7$\tabularnewline*
~~$HPA:042030_12$&hpa&skin 1; keratinocytes[Supporte&$8.13e-03$&$ 10$&$  6$\tabularnewline*
~~$HPA:055010_01$&hpa&vagina; squamous epithelial ce&$1.07e-03$&$ 10$&$  9$\tabularnewline*
~~$HPA:055010_02$&hpa&vagina; squamous epithelial ce&$1.47e-05$&$ 10$&$  9$\tabularnewline*
~~$HPA:055010_03$&hpa&vagina; squamous epithelial ce&$3.00e-02$&$ 10$&$  5$\tabularnewline*
~~$HPA:055010_11$&hpa&vagina; squamous epithelial ce&$7.76e-06$&$ 10$&$  9$\tabularnewline*
~~$HPA:055010_12$&hpa&vagina; squamous epithelial ce&$1.99e-07$&$ 10$&$  9$\tabularnewline*
~~$HPA:055010_13$&hpa&vagina; squamous epithelial ce&$2.11e-02$&$ 10$&$  5$\tabularnewline*
~~$HPA:043010_03$&hpa&skin 2; epidermal cells[Uncert&$9.38e-03$&$ 10$&$  5$\tabularnewline*
~~$HPA:043010_11$&hpa&skin 2; epidermal cells[Suppor&$2.19e-03$&$ 10$&$  8$\tabularnewline*
~~$HPA:043010_12$&hpa&skin 2; epidermal cells[Suppor&$1.46e-02$&$ 10$&$  6$\tabularnewline*
~~$HPA:043010_13$&hpa&skin 2; epidermal cells[Suppor&$9.38e-03$&$ 10$&$  5$\tabularnewline*
~~$OMIM:131800$&omi&EPIDERMOLYSIS BULLOSA SIMPLEX,&$5.00e-02$&$ 10$&$  2$\tabularnewline*
~~$OMIM:601001$&omi&EPIDERMOLYSIS BULLOSA SIMPLEX,&$5.00e-02$&$ 10$&$  2$\tabularnewline*
~~$OMIM:131760$&omi&EPIDERMOLYSIS BULLOSA SIMPLEX,&$5.00e-02$&$ 10$&$  2$\tabularnewline*
~~$OMIM:131900$&omi&EPIDERMOLYSIS BULLOSA SIMPLEX,&$5.00e-02$&$ 10$&$  2$\tabularnewline*
~~$REAC:1266738$&rea&Developmental Biology&$8.62e-07$&$ 10$&$  8$\tabularnewline*
~~$TF:M07051_1$&tf&Factor: NF-1B; motif: CTGGCASG&$6.22e-03$&$ 10$&$  7$\tabularnewline
\hline
\end{longtable}}

%% file: table.top.genes.pool.tex
\begin{table}[!tbp]
{\tiny
\begin{center}
\begin{tabular}{lclclclcl}
\hline\hline
\multicolumn{1}{c}{\bfseries Gray}&\multicolumn{1}{c}{\bfseries }&\multicolumn{1}{c}{\bfseries Olivegreen}&\multicolumn{1}{c}{\bfseries }&\multicolumn{1}{c}{\bfseries Skyblue}&\multicolumn{1}{c}{\bfseries }&\multicolumn{1}{c}{\bfseries Orchid}&\multicolumn{1}{c}{\bfseries }&\multicolumn{1}{c}{\bfseries Coral}\tabularnewline
\cline{1-9}
\multicolumn{1}{c}{n = 139}&\multicolumn{1}{c}{}&\multicolumn{1}{c}{n = 47}&\multicolumn{1}{c}{}&\multicolumn{1}{c}{n = 56}&\multicolumn{1}{c}{}&\multicolumn{1}{c}{n = 48}&\multicolumn{1}{c}{}&\multicolumn{1}{c}{n = 10}\tabularnewline
\hline
MYBL1&&FCER1G&&COL5A2&&CD2&&KRT6A\tabularnewline
BATF&&C1QB&&COL1A2&&PTGDS&&SPRR1A\tabularnewline
STAP1&&IDO1&&COL3A1&&GIMAP4&&SPRR1B\tabularnewline
MME&&GBP1&&VCAN&&ADAMDEC1&&SPRR3\tabularnewline
CD44&&C1QA&&DCN&&CD3D&&S100A2\tabularnewline
CYB5R2&&CD14&&COL6A3&&CCL19&&KRT13\tabularnewline
TNFRSF13B&&GZMA&&THBS2&&IL18&&KRT14\tabularnewline
LRMP&&SERPING1&&SPARC&&TFEC&&DSP\tabularnewline
MARCKSL1&&RARRES3&&SULF1&&ITK&&KRT5\tabularnewline
BCL2A1&&CXCL10&&MMP2&&PLA2G2D&&\tabularnewline
HCK&&PSTPIP2&&MXRA5&&APOC1&&\tabularnewline
CCND2&&GBP2&&LUM&&CHI3L1&&\tabularnewline
VPREB3&&FGL2&&CTGF&&LYZ&&\tabularnewline
LMO2&&CXCL11&&COL15A1&&ENPP2&&\tabularnewline
HLA-DOB&&CCL8&&COL5A1&&LGALS2&&\tabularnewline
STAG3&&LILRB2&&FAP&&CSTA&&\tabularnewline
PDGFD&&CXCL9&&COL1A1&&CXCL13&&\tabularnewline
CCR7&&CD163&&POSTN&&ITM2A&&\tabularnewline
BLNK&&GZMB&&TMEM45A&&CLU&&\tabularnewline
SORL1&&GZMH&&EMP1&&PLA2G7&&\tabularnewline
MNDA&&GZMK&&CTSK&&IL7R&&\tabularnewline
RRAS2&&ALDH1A1&&PLS3&&TRBC2&&\tabularnewline
SPINK2&&IFNG&&TGFBI&&HSD11B1&&\tabularnewline
BACH2&&SLAMF7&&GJA1&&MMP9&&\tabularnewline
NCF2&&CPVL&&COL11A1&&C3&&\tabularnewline
GPR183&&KCNJ2&&AEBP1&&CXCL14&&\tabularnewline
OSBPL10&&CD8A&&TIMP1&&CYP27B1&&\tabularnewline
GRHPR&&MS4A4A&&TNFAIP6&&CHIT1&&\tabularnewline
DUSP5&&MT1G&&ANXA1&&LAMP3&&\tabularnewline
ALOX5AP&&LGMN&&TAGLN&&CCL21&&\tabularnewline
CD22&&MT2A&&FOS&&ROBO1&&\tabularnewline
MS4A1&&IGSF6&&CILP&&MAL&&\tabularnewline
SYBU&&S100A8&&DPT&&KLRB1&&\tabularnewline
TCL1A&&CRTAM&&MGP&&SQOR&&\tabularnewline
FCMR&&GNLY&&SPP1&&ORM1&&\tabularnewline
GPR137B&&S100A9&&G0S2&&SELENOP&&\tabularnewline
IGHM&&GPX3&&STEAP1&&P2RY14&&\tabularnewline
SLC12A8&&MT1M&&MMP1&&NPY1R&&\tabularnewline
CD83&&PLTP&&EPS8&&ORM2&&\tabularnewline
GMDS&&MARCO&&GREM1&&TRDC&&\tabularnewline
\hline
\end{tabular}
\caption{The identified modules from the Pool method, their sizes, and member genes. The genes are sorted decreasingly by their intra-module connectivity (sum of the incident edge weights). Only the top 40 genes are shown.\label{tab:top.genes.pool}}\end{center}}
\end{table}

%% file: table.enrichment.pool.tex
\setlongtables{\tiny
\begin{longtable}{lllrrr}\caption{The significant terms for the gene enrichment ,
                                          analysis of the DLBCL Pool method modules. Number of genes ,
                                          in each term (N), and the overlap to module (O).} \tabularnewline
\hline\hline
\multicolumn{1}{l}{Term ID}&\multicolumn{1}{c}{Domain}&\multicolumn{1}{c}{Term}&\multicolumn{1}{c}{P}&\multicolumn{1}{c}{N}&\multicolumn{1}{c}{O}\tabularnewline
\hline
\endfirsthead\caption[]{\em (continued)} \tabularnewline*
\hline
\multicolumn{1}{l}{Term ID}&\multicolumn{1}{c}{Domain}&\multicolumn{1}{c}{Term}&\multicolumn{1}{c}{P}&\multicolumn{1}{c}{N}&\multicolumn{1}{c}{O}\tabularnewline
\hline
\endhead
\hline
\endfoot
\label{tab:enrichment.pool}
{\bfseries Gray}&&&&&\tabularnewline*
~~$TF:M00940_1$&tf&Factor: E2F-1; motif: NTTTCGCG&$3.04e-02$&$136$&$31$\tabularnewline
\hline
{\bfseries Olivegreen}&&&&&\tabularnewline*
~~$GO:0006952$&BP&defense response&$2.10e-03$&$ 45$&$27$\tabularnewline*
~~$GO:0045087$&BP&innate immune response&$5.53e-04$&$ 45$&$19$\tabularnewline*
~~$GO:0045088$&BP&regulation of innate immune re&$9.89e-03$&$ 45$&$10$\tabularnewline*
~~$REAC:5660526$&rea&Response to metal ions&$4.98e-02$&$ 45$&$ 4$\tabularnewline*
~~$REAC:5661231$&rea&Metallothioneins bind metals&$4.98e-02$&$ 45$&$ 4$\tabularnewline
\hline
{\bfseries Skyblue}&&&&&\tabularnewline*
~~$GO:0051093$&BP&negative regulation of develop&$2.67e-02$&$ 56$&$14$\tabularnewline*
~~$GO:0007167$&BP&enzyme linked receptor protein&$8.80e-03$&$ 56$&$14$\tabularnewline*
~~$GO:0007178$&BP&transmembrane receptor protein&$2.54e-03$&$ 56$&$ 8$\tabularnewline*
~~$GO:0007517$&BP&muscle organ development&$4.98e-02$&$ 56$&$ 9$\tabularnewline*
~~$GO:0010243$&BP&response to organonitrogen com&$1.71e-02$&$ 56$&$12$\tabularnewline*
~~$GO:0043200$&BP&response to amino acid&$1.32e-02$&$ 56$&$ 7$\tabularnewline*
~~$GO:0032501$&BP&multicellular organismal proce&$2.68e-02$&$ 56$&$45$\tabularnewline*
~~$GO:0044707$&BP&single-multicellular organism &$4.07e-02$&$ 56$&$44$\tabularnewline*
~~$GO:0001503$&BP&ossification&$4.02e-03$&$ 56$&$12$\tabularnewline*
~~$GO:0044712$&BP&single-organism catabolic proc&$2.06e-03$&$ 56$&$17$\tabularnewline*
~~$GO:0044236$&BP&multicellular organism metabol&$9.61e-06$&$ 56$&$14$\tabularnewline*
~~$GO:0044243$&BP&multicellular organismal catab&$6.82e-05$&$ 56$&$12$\tabularnewline*
~~$GO:0044259$&BP&multicellular organismal macro&$9.61e-06$&$ 56$&$14$\tabularnewline*
~~$GO:0032963$&BP&collagen metabolic process&$9.61e-06$&$ 56$&$14$\tabularnewline*
~~$GO:0030574$&BP&collagen catabolic process&$6.82e-05$&$ 56$&$12$\tabularnewline*
~~$GO:0070848$&BP&response to growth factor&$1.28e-04$&$ 56$&$15$\tabularnewline*
~~$GO:0071363$&BP&cellular response to growth fa&$5.43e-04$&$ 56$&$14$\tabularnewline*
~~$GO:0009653$&BP&anatomical structure morphogen&$8.62e-03$&$ 56$&$25$\tabularnewline*
~~$GO:0048646$&BP&anatomical structure formation&$1.54e-02$&$ 56$&$17$\tabularnewline*
~~$GO:0072359$&BP&circulatory system development&$8.14e-05$&$ 56$&$19$\tabularnewline*
~~$GO:0072358$&BP&cardiovascular system developm&$3.32e-04$&$ 56$&$17$\tabularnewline*
~~$GO:0001944$&BP&vasculature development&$3.32e-04$&$ 56$&$17$\tabularnewline*
~~$GO:0009887$&BP&animal organ morphogenesis&$2.44e-03$&$ 56$&$14$\tabularnewline*
~~$GO:0001568$&BP&blood vessel development&$3.32e-04$&$ 56$&$17$\tabularnewline*
~~$GO:0048514$&BP&blood vessel morphogenesis&$4.75e-03$&$ 56$&$14$\tabularnewline*
~~$GO:0001525$&BP&angiogenesis&$1.69e-02$&$ 56$&$13$\tabularnewline*
~~$GO:0071822$&BP&protein complex subunit organi&$1.17e-02$&$ 56$&$16$\tabularnewline*
~~$GO:0071840$&BP&cellular component organizatio&$2.24e-03$&$ 56$&$38$\tabularnewline*
~~$GO:0016043$&BP&cellular component organizatio&$1.71e-03$&$ 56$&$38$\tabularnewline*
~~$GO:0097435$&BP&supramolecular fiber organizat&$1.28e-04$&$ 56$&$15$\tabularnewline*
~~$GO:0043062$&BP&extracellular structure organi&$1.19e-13$&$ 56$&$25$\tabularnewline*
~~$GO:0030198$&BP&extracellular matrix organizat&$1.19e-13$&$ 56$&$25$\tabularnewline*
~~$GO:0030199$&BP&collagen fibril organization&$5.59e-05$&$ 56$&$ 9$\tabularnewline*
~~$GO:0009888$&BP&tissue development&$3.46e-02$&$ 56$&$22$\tabularnewline*
~~$GO:0061448$&BP&connective tissue development&$2.97e-03$&$ 56$&$11$\tabularnewline*
~~$GO:0001501$&BP&skeletal system development&$1.84e-09$&$ 56$&$19$\tabularnewline*
~~$GO:0051216$&BP&cartilage development&$4.43e-04$&$ 56$&$10$\tabularnewline*
~~$GO:0009611$&BP&response to wounding&$6.33e-04$&$ 56$&$17$\tabularnewline*
~~$GO:0042060$&BP&wound healing&$4.54e-03$&$ 56$&$15$\tabularnewline*
~~$GO:0009719$&BP&response to endogenous stimulu&$4.15e-05$&$ 56$&$19$\tabularnewline*
~~$GO:0071495$&BP&cellular response to endogenou&$2.93e-04$&$ 56$&$15$\tabularnewline*
~~$GO:0005576$&CC&extracellular region&$2.53e-04$&$ 56$&$46$\tabularnewline*
~~$GO:0044421$&CC&extracellular region part&$2.09e-06$&$ 56$&$46$\tabularnewline*
~~$GO:0005615$&CC&extracellular space&$1.76e-06$&$ 56$&$38$\tabularnewline*
~~$GO:0031012$&CC&extracellular matrix&$8.36e-11$&$ 56$&$25$\tabularnewline*
~~$GO:0044420$&CC&extracellular matrix component&$3.76e-08$&$ 56$&$13$\tabularnewline*
~~$GO:0005578$&CC&proteinaceous extracellular ma&$5.45e-12$&$ 56$&$23$\tabularnewline*
~~$GO:0005604$&CC&basement membrane&$1.94e-03$&$ 56$&$ 7$\tabularnewline*
~~$GO:0005581$&CC&collagen trimer&$3.40e-04$&$ 56$&$11$\tabularnewline*
~~$GO:0098644$&CC&complex of collagen trimers&$1.94e-03$&$ 56$&$ 7$\tabularnewline*
~~$GO:0098643$&CC&banded collagen fibril&$1.94e-03$&$ 56$&$ 7$\tabularnewline*
~~$GO:0005583$&CC&fibrillar collagen trimer&$1.94e-03$&$ 56$&$ 7$\tabularnewline*
~~$GO:0012505$&CC&endomembrane system&$1.71e-02$&$ 56$&$33$\tabularnewline*
~~$GO:0005788$&CC&endoplasmic reticulum lumen&$2.54e-03$&$ 56$&$ 8$\tabularnewline*
~~$GO:0032403$&MF&protein complex binding&$3.00e-02$&$ 56$&$13$\tabularnewline*
~~$GO:0043167$&MF&ion binding&$4.41e-02$&$ 56$&$32$\tabularnewline*
~~$GO:0043169$&MF&cation binding&$1.24e-02$&$ 56$&$24$\tabularnewline*
~~$GO:0046872$&MF&metal ion binding&$4.34e-03$&$ 56$&$24$\tabularnewline*
~~$GO:0005201$&MF&extracellular matrix structura&$8.67e-05$&$ 56$&$10$\tabularnewline*
~~$GO:0050840$&MF&extracellular matrix binding&$1.11e-02$&$ 56$&$ 6$\tabularnewline*
~~$KEGG:04933$&keg&AGE-RAGE signaling pathway in &$1.62e-02$&$ 56$&$ 6$\tabularnewline*
~~$KEGG:04974$&keg&Protein digestion and absorpti&$3.86e-02$&$ 56$&$ 8$\tabularnewline*
~~$REAC:1474244$&rea&Extracellular matrix organizat&$7.33e-07$&$ 56$&$18$\tabularnewline*
~~$REAC:1474290$&rea&Collagen formation&$6.62e-03$&$ 56$&$ 9$\tabularnewline*
~~$REAC:1650814$&rea&Collagen biosynthesis and modi&$6.07e-05$&$ 56$&$ 9$\tabularnewline*
~~$REAC:2022090$&rea&Assembly of collagen fibrils a&$2.10e-02$&$ 56$&$ 8$\tabularnewline
\hline
{\bfseries Orchid}&&&&&\tabularnewline*
~~$GO:0042581$&CC&specific granule&$4.55e-02$&$ 48$&$ 6$\tabularnewline*
~~$GO:0035580$&CC&specific granule lumen&$4.55e-02$&$ 48$&$ 6$\tabularnewline
\hline
\newpage
{\bfseries Coral}&&&&&\tabularnewline*
~~$GO:0009888$&BP&tissue development&$1.02e-03$&$ 10$&$ 9$\tabularnewline*
~~$GO:0008544$&BP&epidermis development&$2.23e-10$&$ 10$&$ 9$\tabularnewline*
~~$GO:0060429$&BP&epithelium development&$7.35e-07$&$ 10$&$ 9$\tabularnewline*
~~$GO:0030855$&BP&epithelial cell differentiatio&$1.01e-08$&$ 10$&$ 9$\tabularnewline*
~~$GO:0009913$&BP&epidermal cell differentiation&$2.47e-11$&$ 10$&$ 9$\tabularnewline*
~~$GO:0008219$&BP&cell death&$3.95e-03$&$ 10$&$ 9$\tabularnewline*
~~$GO:0012501$&BP&programmed cell death&$2.22e-03$&$ 10$&$ 9$\tabularnewline*
~~$GO:0043588$&BP&skin development&$5.35e-09$&$ 10$&$ 9$\tabularnewline*
~~$GO:0030216$&BP&keratinocyte differentiation&$2.47e-11$&$ 10$&$ 9$\tabularnewline*
~~$GO:0031424$&BP&keratinization&$1.13e-12$&$ 10$&$ 9$\tabularnewline*
~~$GO:0070268$&BP&cornification&$1.44e-10$&$ 10$&$ 8$\tabularnewline*
~~$GO:0018149$&BP&peptide cross-linking&$3.55e-02$&$ 10$&$ 4$\tabularnewline*
~~$GO:0001533$&CC&cornified envelope&$4.44e-03$&$ 10$&$ 4$\tabularnewline*
~~$GO:0099513$&CC&polymeric cytoskeletal fiber&$4.54e-03$&$ 10$&$ 5$\tabularnewline*
~~$GO:0045111$&CC&intermediate filament cytoskel&$3.76e-05$&$ 10$&$ 5$\tabularnewline*
~~$GO:0005882$&CC&intermediate filament&$6.36e-06$&$ 10$&$ 5$\tabularnewline*
~~$GO:0045095$&CC&keratin filament&$3.06e-04$&$ 10$&$ 4$\tabularnewline*
~~$GO:0005198$&MF&structural molecule activity&$4.41e-05$&$ 10$&$ 8$\tabularnewline*
~~$GO:0005200$&MF&structural constituent of cyto&$2.00e-02$&$ 10$&$ 4$\tabularnewline*
~~$HPA:009020_01$&hpa&cervix, uterine; squamous epit&$1.46e-02$&$ 10$&$ 8$\tabularnewline*
~~$HPA:009020_02$&hpa&cervix, uterine; squamous epit&$3.99e-04$&$ 10$&$ 8$\tabularnewline*
~~$HPA:009020_03$&hpa&cervix, uterine; squamous epit&$1.30e-03$&$ 10$&$ 6$\tabularnewline*
~~$HPA:009020_11$&hpa&cervix, uterine; squamous epit&$3.02e-04$&$ 10$&$ 8$\tabularnewline*
~~$HPA:009020_12$&hpa&cervix, uterine; squamous epit&$1.33e-05$&$ 10$&$ 8$\tabularnewline*
~~$HPA:009020_13$&hpa&cervix, uterine; squamous epit&$8.27e-04$&$ 10$&$ 6$\tabularnewline*
~~$HPA:053030_01$&hpa&tonsil; squamous epithelial ce&$9.03e-03$&$ 10$&$ 9$\tabularnewline*
~~$HPA:053030_02$&hpa&tonsil; squamous epithelial ce&$1.63e-04$&$ 10$&$ 9$\tabularnewline*
~~$HPA:053030_11$&hpa&tonsil; squamous epithelial ce&$7.87e-05$&$ 10$&$ 9$\tabularnewline*
~~$HPA:053030_12$&hpa&tonsil; squamous epithelial ce&$2.70e-06$&$ 10$&$ 9$\tabularnewline*
~~$HPA:055010_01$&hpa&vagina; squamous epithelial ce&$1.07e-03$&$ 10$&$ 9$\tabularnewline*
~~$HPA:055010_02$&hpa&vagina; squamous epithelial ce&$1.47e-05$&$ 10$&$ 9$\tabularnewline*
~~$HPA:055010_03$&hpa&vagina; squamous epithelial ce&$3.00e-02$&$ 10$&$ 5$\tabularnewline*
~~$HPA:055010_11$&hpa&vagina; squamous epithelial ce&$7.76e-06$&$ 10$&$ 9$\tabularnewline*
~~$HPA:055010_12$&hpa&vagina; squamous epithelial ce&$1.99e-07$&$ 10$&$ 9$\tabularnewline*
~~$HPA:055010_13$&hpa&vagina; squamous epithelial ce&$2.11e-02$&$ 10$&$ 5$\tabularnewline*
~~$HPA:029010_01$&hpa&oral mucosa; squamous epitheli&$2.59e-03$&$ 10$&$ 9$\tabularnewline*
~~$HPA:029010_02$&hpa&oral mucosa; squamous epitheli&$8.70e-04$&$ 10$&$ 8$\tabularnewline*
~~$HPA:029010_03$&hpa&oral mucosa; squamous epitheli&$9.38e-03$&$ 10$&$ 5$\tabularnewline*
~~$HPA:029010_11$&hpa&oral mucosa; squamous epitheli&$3.54e-05$&$ 10$&$ 9$\tabularnewline*
~~$HPA:029010_12$&hpa&oral mucosa; squamous epitheli&$8.81e-05$&$ 10$&$ 8$\tabularnewline*
~~$HPA:029010_13$&hpa&oral mucosa; squamous epitheli&$9.38e-03$&$ 10$&$ 5$\tabularnewline*
~~$HPA:043010_03$&hpa&skin 2; epidermal cells[Uncert&$9.38e-03$&$ 10$&$ 5$\tabularnewline*
~~$HPA:043010_11$&hpa&skin 2; epidermal cells[Suppor&$2.19e-03$&$ 10$&$ 8$\tabularnewline*
~~$HPA:043010_12$&hpa&skin 2; epidermal cells[Suppor&$1.46e-02$&$ 10$&$ 6$\tabularnewline*
~~$HPA:043010_13$&hpa&skin 2; epidermal cells[Suppor&$9.38e-03$&$ 10$&$ 5$\tabularnewline*
~~$HPA:015010_01$&hpa&esophagus; squamous epithelial&$3.06e-03$&$ 10$&$ 9$\tabularnewline*
~~$HPA:015010_02$&hpa&esophagus; squamous epithelial&$1.29e-04$&$ 10$&$ 9$\tabularnewline*
~~$HPA:015010_03$&hpa&esophagus; squamous epithelial&$9.55e-05$&$ 10$&$ 7$\tabularnewline*
~~$HPA:015010_11$&hpa&esophagus; squamous epithelial&$4.66e-05$&$ 10$&$ 9$\tabularnewline*
~~$HPA:015010_12$&hpa&esophagus; squamous epithelial&$1.84e-06$&$ 10$&$ 9$\tabularnewline*
~~$HPA:015010_13$&hpa&esophagus; squamous epithelial&$3.52e-05$&$ 10$&$ 7$\tabularnewline*
~~$HPA:042030_11$&hpa&skin 1; keratinocytes[Supporte&$1.62e-02$&$ 10$&$ 7$\tabularnewline*
~~$HPA:042030_12$&hpa&skin 1; keratinocytes[Supporte&$8.13e-03$&$ 10$&$ 6$\tabularnewline*
~~$OMIM:601001$&omi&EPIDERMOLYSIS BULLOSA SIMPLEX,&$5.00e-02$&$ 10$&$ 2$\tabularnewline*
~~$OMIM:131760$&omi&EPIDERMOLYSIS BULLOSA SIMPLEX,&$5.00e-02$&$ 10$&$ 2$\tabularnewline*
~~$OMIM:131900$&omi&EPIDERMOLYSIS BULLOSA SIMPLEX,&$5.00e-02$&$ 10$&$ 2$\tabularnewline*
~~$OMIM:131800$&omi&EPIDERMOLYSIS BULLOSA SIMPLEX,&$5.00e-02$&$ 10$&$ 2$\tabularnewline*
~~$REAC:1266738$&rea&Developmental Biology&$8.62e-07$&$ 10$&$ 8$\tabularnewline*
~~$TF:M07051_1$&tf&Factor: NF-1B; motif: CTGGCASG&$6.22e-03$&$ 10$&$ 7$\tabularnewline
\hline
\end{longtable}}